\documentclass{article}
\usepackage{log_2026}
\usepackage{booktabs}
\usepackage{placeins}
\usepackage{amsfonts}
\usepackage{enumitem}
\usepackage{float}
\usepackage{amsthm}
\usepackage{makecell}
\usepackage{colortbl}
\usepackage{graphicx}
\usepackage[natbib=true]{biblatex}
\graphicspath{{figures/}}
\newtheorem{definition}{Definition}
\newtheorem{lemma}{Lemma}

\definecolor{posgreen}{RGB}{0,128,0}
\definecolor{negred}{RGB}{200,0,0}
\newcommand{\posgreen}[1]{\textcolor{posgreen}{#1}}
\newcommand{\negred}[1]{\textcolor{negred}{#1}}

\definecolor{linkcolor}{rgb}{0.6, 0.9, 1.0}
\AtEndPreamble{%
  \hypersetup{
    colorlinks=false,
    allbordercolors=linkcolor,
    pdfborderstyle={/S/U/W 0.5},
    allcolors=black
  }
}

\definecolor{customcyan}{RGB}{10, 204, 0}
\definecolor{tealblue}{RGB}{0, 132, 194}
\definecolor{darkorange}{RGB}{220, 100, 0}

\newsavebox{\tabbox}
\newcommand{\fittotextwidth}[1]{%
  \sbox{\tabbox}{#1}%
  \ifdim\wd\tabbox>\linewidth
    \resizebox{\linewidth}{!}{\usebox{\tabbox}}%
  \else
    \usebox{\tabbox}%
  \fi
}

\title[Graph Transductive Sharpening]{Graph Transductive Sharpening: Leveraging Unlabeled Predictions in Node Classification}

\author[B. Zaz, M. Gonzàlez I Català, F. Hernandez Caralt, M. Eliasof \& P. Liò]{%
  Brown Zaz\thanks{Equal contribution.}\\
  University of Cambridge\\
  \email{jb2650@cl.cam.ac.uk} \\
  \And
  Mar Gonzàlez I Català\footnotemark[1]\\
  University of Cambridge\\
  \email{mg2211@cl.cam.ac.uk} \\
  \And
  Ferran Hernandez Caralt \\
  University of Cambridge \\
  \email{fh455@cl.cam.ac.uk} \\
  \And
  Moshe Eliasof \\
  University of Cambridge \\
  \email{me532@cl.cam.ac.uk} \\
  \And
  Pietro Liò \\
  University of Cambridge \\
  \email{pl219@cl.cam.ac.uk} \\
}
\begin{document}

\maketitle

\begin{abstract}
In the transductive setting, where the full graph is observed but node labels are only partially available, progress in semi-supervised node classification has largely focused on architectural innovation. In this paper, we revisit an orthogonal axis: the training objective. We start from a simple observation: transductive models produce predictions for every node during training, including nodes without labels. These unlabeled-node predictions may contain useful training signal, but standard supervised objectives discard them because no ground-truth labels are available. Inspired by the decomposition of cross-entropy into a label-dependent alignment term and a label-independent entropy term, we propose prediction confidence as a natural way to extract this signal in the absence of labels. This motivates Transductive Sharpening (TS): a loss-level modification that minimizes prediction entropy on unlabeled nodes while counterbalancing this effect on labeled nodes. We evaluate Transductive Sharpening across a wide range of node-classification benchmarks and observe that it improves or matches the corresponding baseline without requiring any changes to the backbone architecture. Code is available at \url{https://github.com/transductive-sharpening/tunedGNN}.
\end{abstract}

\section{Introduction}

Graph neural networks (GNNs)~\cite{548916, 1555942, bruna2014spectralnetworkslocallyconnected, defferrard2017convolutionalneuralnetworksgraphs,
kipf2017semisupervisedclassificationgraphconvolutional, gat, hamilton2018inductiverepresentationlearninglarge, gilmer2017neuralmessagepassingquantum} have become the dominant approach for node classification tasks, particularly in the transductive setting~\cite{bronstein2021geometric}, where the full graph is observed but only a subset of node labels is available. Over the past years, progress in this area has been driven largely by architectural innovation, with increasingly sophisticated message-passing schemes~\cite{dir-poly, dual-former, bunn, gcn-rep,luan2023graph} and transformer-based models~\cite{graph-tarif}. In contrast, the design of training objectives has received comparatively little attention, despite its central role in shaping model performance.

In the transductive setting, models produce predictions for all nodes in the graph at every training step, including those without labels, yet the training objective is applied only to labeled nodes, as standard supervised losses require ground-truth labels. However, once a model begins to form reliable and confident predictions, these may themselves provide a useful learning signal for nodes with unknown labels. We build on this observation by leveraging such predictions, encouraging confidence on unlabeled nodes while preventing overconfidence on labeled ones.

We introduce Transductive Sharpening (TS), a simple loss-level modification that implements this idea. TS combines entropy minimization on unlabeled nodes~\cite{ssl-entropy-minimization} with entropy maximization on labeled nodes~\cite{pereyra2017regularizing}. The method introduces a single hyperparameter and can be applied on top of any GNN architecture. Empirically, we show that it improves or matches the corresponding baseline across a wide variety of models and benchmarks.

To understand the behavior induced by TS, we study the effect of the sharpening coefficient $\lambda$. Empirically, we find that moderate positive values of $\lambda$ yield the most reliable gains, and that setting $\lambda = 0.25$ transfers well across domains, yielding competitive performance across most datasets and architectures without domain-specific tuning.

Our results point to a simple but powerful principle: predictions generated during training, typically discarded when labels are unavailable, can be directly leveraged to improve learning.

\paragraph{Contributions.}
Our main contributions are as follows:
\begin{itemize}[leftmargin=1.5em]
    \item We introduce \emph{Transductive Sharpening} (TS), a simple architecture-agnostic loss modification that turns predictions on unlabeled nodes into a direct training signal for transductive node classification.
    \item We show that TS provides a strong performance-complexity trade-off: it improves or matches standard GNN and MLP baselines across 13 node-classification benchmarks while adding only a single scalar hyperparameter and requiring no architectural changes.
    \item We study the role of the sharpening coefficient $\lambda$, showing that TS remains effective across a range of positive values, and that a single conservative setting preserves much of the benefit across models and datasets.
\end{itemize}

\section{Background and Setup}

In this section we provide background material related to our work.

\paragraph{Notation.}
We denote by $\Delta^{C-1} = \{p \in \mathbb{R}^C_{\ge 0} : \sum_{i=1}^C p_i = 1\}$ the probability simplex over $C$ classes. Throughout, labels are represented as one-hot vectors $y_v \in \{0,1\}^C$, where $y_{v,i}$ indicates whether node $v$ belongs to class $i$.

Node classification tasks consist of assigning a label to each node in a graph based on its features and the graph structure~\cite{kipf2017semisupervisedclassificationgraphconvolutional}.
In the transductive setting, the full graph and node features are available during training, but labels are observed only for a subset of nodes, and the goal is to predict the rest.

\begin{definition}[Transductive node classification]
Let $G = (V, E)$ be a graph with node features $X \in \mathbb{R}^{|V| \times d}$. Each node $v \in V$ has an associated label $y_v \in \{0,1\}^C$, observed only for a subset $V_L \subset V$ (referred to as \emph{labeled nodes}). We denote by $V_U := V \setminus V_L$ the remaining nodes (referred to as \emph{unlabeled nodes}). The objective is to learn a model that predicts labels for nodes in $V_U$, using the full graph $G$, all node features $X$, and the labels observed on $V_L$.
\end{definition}

\textbf{From standard to augmented training objectives.}
A common approach to transductive node classification is to train a model in a supervised manner on the labeled subset of nodes, and then use it to generate predictions for the unlabeled nodes.

In practice, a model produces, for each node $v \in V$, a logit vector $z_v \in \mathbb{R}^C$ and a corresponding probability distribution
\(
p_v = \mathrm{softmax}(z_v) \in \Delta^{C-1}
\).

Models are trained using cross-entropy loss applied only to labeled nodes:
\begin{equation}
\mathcal{L}_{\mathrm{sup}} = - \frac{1}{|V_L|} \sum_{v \in V_L} \sum_{i=1}^C y_{v,i} \log p_{v,i}.
\end{equation}

This objective aligns the model's predictions with ground-truth labels, but ignores the feature-based outputs generated on unlabeled nodes $V_U$, even though these are computed at every training step as the model processes the full graph.

Although the loss cannot be evaluated on unlabeled nodes due to the absence of labels, the model's predictions produced during training may still contain useful information that could be exploited. This suggests augmenting the objective with an additional term defined over unlabeled nodes similarly to~\cite{ssl-entropy-minimization}.

\begin{definition}[Augmented transductive objective]
\label{def:transductive-regularization}
Consider a transductive node classification problem with labeled nodes $V_L$ and unlabeled nodes $V_U$. Let $p_v \in \Delta^{C-1}$ denote the predictive distribution for node $v$. An augmented transductive objective is any training objective of the form
\begin{equation}
\mathcal{L}
=
\mathcal{L}_{\mathrm{sup}}
+
f\bigl(\{p_v\}_{v \in V_U}\bigr),
\end{equation}
where $f : (\Delta^{C-1})^{|V_U|} \to \mathbb{R}$ extracts a learning signal from the model's predictions on the unlabeled nodes $V_U$.
\end{definition}

Definition~\ref{def:transductive-regularization} highlights the flexibility of this framework: different choices of $f$ induce different ways of extracting learning signals from unlabeled predictions. For instance, in a binary classification problem with a known balanced class distribution, $f$ could penalize deviations from a balanced prediction distribution over the unlabeled nodes during training.

While useful, this example depends on information that may not be available in general. The central question, then, is whether we can choose $f$ in a principled, task-agnostic way that applies broadly across transductive node classification problems.

\textbf{Uncertainty-based learning signals.}
A natural starting point for designing $f$ is to examine the structure of the supervised loss itself. In particular, cross-entropy admits the following decomposition:

\begin{lemma}[Cross-entropy decomposition]
\label{lemma:cross-entropy}
For any target distribution $y$ and prediction $p$, cross-entropy loss can be written as
\begin{equation}
\mathcal{L}_{\mathrm{CE}}(y, p)
=
H(p)
+
\sum_{i=1}^C (p_i - y_i)\log p_i,
\end{equation}
where $H(p) = -\sum_i p_i \log p_i$ denotes the Shannon entropy~\cite{shannon_entropy}.
\end{lemma}

The proof of Lemma~\ref{lemma:cross-entropy} is provided in Appendix~\ref{app:proof_cross_entropy}.

The second term depends explicitly on the target labels\footnote{For a Bayesian perspective on this approach refer to~\cite{ssl-entropy-minimization}}, whereas the entropy term $H(p)$ depends only on the model's predictions. This separation reveals that part of the supervised objective is inherently label-independent, and can therefore be evaluated on any node.

Lemma~\ref{lemma:cross-entropy} motivates a class of choices for $f$: functions that depend only on the predictive distribution and capture uncertainty-related properties of the model's predictions. In this work, we investigate the empirical effectiveness of this approach through a simple instantiation of this perspective that explicitly controls prediction confidence across labeled and unlabeled nodes via the training objective.

\section{Transductive Sharpening for Graph Learning}

In this section, we introduce \emph{Transductive Sharpening} (TS), a loss-level modification for transductive node classification.

\subsection{Transductive Sharpening Objective}

Lemma~\ref{lemma:cross-entropy} motivates a label-free way to use unlabeled nodes: shape the uncertainty of their predictive distributions directly through the objective. We instantiate this principle by adding an uncertainty term over $V_U$ to the supervised loss.

A natural choice for the unlabeled-node term is to encourage low-uncertainty predictions on unlabeled nodes, without otherwise modifying the supervised objective. However, naively minimizing uncertainty everywhere can lead to overconfident and poorly calibrated models. To address this, we introduce a simple symmetric objective that sharpens predictions on unlabeled nodes while counterbalancing this effect on labeled ones.

\begin{definition}[Generic Transductive Sharpening objective]
\label{def:generic-TS}
Let \(R : \Delta^{C-1} \to \mathbb{R}\) be a function on the probability simplex that measures the uncertainty of a predictive distribution. For a model producing, for each node \(v \in V\), a probability vector \(p_v \in \Delta^{C-1}\), we define the Generic Transductive Sharpening objective by
\begin{equation}
\mathcal{L}_{R}
=
\mathcal{L}_{\mathrm{sup}}
+
\lambda \cdot \frac{1}{|V_U|} \sum_{v \in V_U} R(p_v)
-
\lambda \cdot \frac{1}{|V_L|} \sum_{v \in V_L} R(p_v),
\end{equation}
where \(\lambda \in \mathbb{R}\) controls the influence of the sharpening.
\end{definition}

This formulation captures two complementary effects. On unlabeled nodes, minimizing uncertainty encourages confident predictions~\cite{ssl-entropy-minimization}, allowing the model to leverage its own outputs as a learning signal. On labeled nodes, maximizing uncertainty counteracts overconfidence, helping to prevent overfitting to the training data~\cite{pereyra2017regularizing,muller2019does}.

\subsection{Implementing an Uncertainty Function}

The formulation introduced in Definition~\ref{def:generic-TS} depends on the choice of the function \(R\), which determines how predictive confidence is shaped during training. A natural candidate for \(R\) is Shannon entropy~\cite{shannon_entropy}, widely used as a measure of uncertainty in probabilistic models.
However, in this work we consider an alternative based on the Tsallis entropy, which offers a convenient quadratic, log-free form.

\begin{definition}[Tsallis entropy with \(q=2\)]
The Tsallis entropy~\cite{Tsallis:1987eu} of order \(q=2\), also known as Gini impurity, is defined as
\begin{equation}
S_2(p) = 1 - \sum_{i=1}^{C} p_i^2.
\end{equation}
\end{definition}

Tsallis entropy is a one-parameter generalization of Shannon entropy that reduces to Shannon entropy when \(q = 1\). For \(q=2\), it preserves the same qualitative behavior, assigning low values to confident predictions and high values to diffuse predictions, but admits a simple quadratic form
\[
S_2(p) = 1 - \|p\|_2^2.
\]
Thus, minimizing \(S_2(p)\) is equivalent to maximizing the squared \(\ell_2\)-norm of the predictive distribution.

We adopt \(R(p) = S_2(p)\) as a simple practical instantiation of the generic TS objective. Its quadratic form is straightforward to compute and differentiate, without requiring additional optimization techniques. Substituting \(R(p) = S_2(p)\) into Definition~\ref{def:generic-TS} yields the final form of the Transductive Sharpening objective:

\begin{definition}[Transductive Sharpening Objective]
For a model producing, for each node \(v \in V\), a probability vector \(p_v \in \Delta^{C-1}\), we define the Transductive Sharpening objective by
\begin{equation}
\mathcal{L}_{\mathrm{TS}}
=
\mathcal{L}_{\mathrm{sup}}
+
\lambda \cdot \frac{1}{|V_U|} \sum_{v \in V_U} \left( 1 - \|p_v\|_2^2 \right)
-
\lambda \cdot \frac{1}{|V_L|} \sum_{v \in V_L} \left( 1 - \|p_v\|_2^2 \right),
\end{equation}
where \(\lambda \in \mathbb{R}\) controls the strength of the sharpening.
\end{definition}

\paragraph{On the choice of sharpening coefficient.}
The Transductive Sharpening objective introduces a single scalar hyperparameter $\lambda$ that controls the strength of the unlabeled-node confidence signal relative to the supervised objective. In our main experiments (Section~\ref{sec:node_class}), we select \(\lambda\) by standard validation tuning over a fixed grid and report test performance at the validation-selected value. To assess whether TS depends on precise per-dataset tuning, we also evaluate a universal setting with $\lambda = 0.25$ applied across all datasets and architectures (Section~\ref{sec:lambda-effect}). Notably, TS requires no changes to the backbone architecture and adds negligible computational overhead to existing GNN training pipelines.

\paragraph{Computational complexity.}
Computing the TS terms and their gradients requires $O(C)$ work per node, so $O(nC)$ additional work per training step, where $n$ is the number of nodes and $C$ is the number of classes. Since the backbone already calculates $O(nC)$ predictions, TS does not change the time or memory complexity.

\section{Experiments}
\label{sec:results}

In this section, we conduct an extensive set of experiments to demonstrate the effectiveness of TS for node classification in transductive settings. Our experiments seek to address the following questions:

\begin{enumerate}[
    label=(Q\arabic*),
    leftmargin=2.8em,
    labelsep=0.8em,
    itemsep=0.4em,
    align=left
]
    \item Does transductive sharpening improve the performance of GNNs across a broad set of node classification benchmarks?
    \item Does transductive sharpening improve the performance of MLPs across a broad set of node classification benchmarks, and does it substitute the effect of message passing?
    \item How does transductive sharpening performance compare to existing node classification methods that use unlabeled nodes?
    \item How does the choice of sharpening coefficient $\lambda$ affect the performance of TS?
    \item Can a fixed $\lambda$ perform competitively across architectures and datasets?
\end{enumerate}

\paragraph{Baselines.} To isolate the effect of TS, we consider three classes of baselines: (i) \textit{GNNs}, using well-tuned implementations of standard message-passing architectures (GCN, GAT, GraphSAGE) following~\cite{tuned-gnn}, (ii) \textit{MLPs}, which operate on node features alone and do not use graph structure, and (iii) \textit{Self-training (ST)}, a node classification method that uses unlabeled nodes. Full comparison to current competitive methods can be found in Appendix~\ref{app:comparison}.

Following Platonov et al.~\cite{platonov2023critical}, Minesweeper is evaluated by ROC-AUC; all other datasets report accuracy. We provide complete details on the experimental settings and datasets in Appendix~\ref{app:experimental_details}, as well as additional experiments in Appendix~\ref{app:additional_results}.

\subsection{Node Classification Results}
\label{sec:node_class}

\begin{table}[H]
    \centering
    \caption{Per-cell results, with the baselines (MLP, GCN, SAGE, GAT) on the left and the corresponding TS-augmented models on the right.  Each treatment cell shows the value on the top line and the $\Delta$ vs.\ the matching baseline on the bottom.  \posgreen{Green} when positive, \negred{red} when negative; \textbf{bold} when $|\Delta| > \sigma$.}
    \label{tab:results}
    \fittotextwidth{%
        \setlength{\tabcolsep}{2pt}%
        \begin{tabular}{r *{4}{c} c *{4}{c}}
        \toprule
         & \multicolumn{4}{c}{Baseline} & & \multicolumn{4}{c}{Baseline+TS (Ours)} \\ \cmidrule(lr){2-5} \cmidrule(lr){7-10}
        Dataset & MLP & GCN & SAGE & GAT & \hspace{0.15em} & MLP & GCN & SAGE & GAT \\
        \midrule
        Cora & \makecell[tc]{60.96 \\[-0.7ex] {\scriptsize$\pm$\,2.51}} & \makecell[tc]{84.54 \\[-0.7ex] {\scriptsize$\pm$\,0.86}} & \makecell[tc]{83.60 \\[-0.7ex] {\scriptsize$\pm$\,0.58}} & \makecell[tc]{82.40 \\[-0.7ex] {\scriptsize$\pm$\,1.01}} & & \makecell[tc]{64.48\,{\scriptsize$\pm$\,2.83} \\ \posgreen{$+$3.52}} & \makecell[tc]{85.74\,{\scriptsize$\pm$\,0.54} \\ \textbf{\posgreen{$+$1.20}}} & \makecell[tc]{85.28\,{\scriptsize$\pm$\,1.11} \\ \textbf{\posgreen{$+$1.68}}} & \makecell[tc]{84.62\,{\scriptsize$\pm$\,0.89} \\ \textbf{\posgreen{$+$2.22}}} \\
        CiteSeer & \makecell[tc]{56.58 \\[-0.7ex] {\scriptsize$\pm$\,1.14}} & \makecell[tc]{72.68 \\[-0.7ex] {\scriptsize$\pm$\,0.43}} & \makecell[tc]{69.60 \\[-0.7ex] {\scriptsize$\pm$\,0.61}} & \makecell[tc]{71.90 \\[-0.7ex] {\scriptsize$\pm$\,0.25}} & & \makecell[tc]{62.72\,{\scriptsize$\pm$\,3.20} \\ \textbf{\posgreen{$+$6.14}}} & \makecell[tc]{75.18\,{\scriptsize$\pm$\,0.15} \\ \textbf{\posgreen{$+$2.50}}} & \makecell[tc]{74.96\,{\scriptsize$\pm$\,0.24} \\ \textbf{\posgreen{$+$5.36}}} & \makecell[tc]{74.84\,{\scriptsize$\pm$\,0.48} \\ \textbf{\posgreen{$+$2.94}}} \\
        PubMed & \makecell[tc]{68.96 \\[-0.7ex] {\scriptsize$\pm$\,1.16}} & \makecell[tc]{80.70 \\[-0.7ex] {\scriptsize$\pm$\,0.96}} & \makecell[tc]{77.86 \\[-0.7ex] {\scriptsize$\pm$\,1.49}} & \makecell[tc]{79.76 \\[-0.7ex] {\scriptsize$\pm$\,1.21}} & & \makecell[tc]{72.30\,{\scriptsize$\pm$\,1.89} \\ \textbf{\posgreen{$+$3.34}}} & \makecell[tc]{80.74\,{\scriptsize$\pm$\,0.30} \\ \posgreen{$+$0.04}} & \makecell[tc]{79.72\,{\scriptsize$\pm$\,0.69} \\ \textbf{\posgreen{$+$1.86}}} & \makecell[tc]{78.84\,{\scriptsize$\pm$\,0.67} \\ \negred{$-$0.92}} \\
        Computer & \makecell[tc]{82.46 \\[-0.7ex] {\scriptsize$\pm$\,0.45}} & \makecell[tc]{94.12 \\[-0.7ex] {\scriptsize$\pm$\,0.08}} & \makecell[tc]{93.25 \\[-0.7ex] {\scriptsize$\pm$\,0.36}} & \makecell[tc]{93.98 \\[-0.7ex] {\scriptsize$\pm$\,0.22}} & & \makecell[tc]{82.87\,{\scriptsize$\pm$\,0.56} \\ \posgreen{$+$0.41}} & \makecell[tc]{93.98\,{\scriptsize$\pm$\,0.26} \\ \negred{$-$0.14}} & \makecell[tc]{93.43\,{\scriptsize$\pm$\,0.11} \\ \posgreen{$+$0.18}} & \makecell[tc]{93.86\,{\scriptsize$\pm$\,0.13} \\ \negred{$-$0.12}} \\
        Photo & \makecell[tc]{87.57 \\[-0.7ex] {\scriptsize$\pm$\,0.52}} & \makecell[tc]{95.90 \\[-0.7ex] {\scriptsize$\pm$\,0.33}} & \makecell[tc]{96.43 \\[-0.7ex] {\scriptsize$\pm$\,0.27}} & \makecell[tc]{96.69 \\[-0.7ex] {\scriptsize$\pm$\,0.14}} & & \makecell[tc]{87.65\,{\scriptsize$\pm$\,0.22} \\ \posgreen{$+$0.08}} & \makecell[tc]{96.21\,{\scriptsize$\pm$\,0.11} \\ \posgreen{$+$0.31}} & \makecell[tc]{96.51\,{\scriptsize$\pm$\,0.23} \\ \posgreen{$+$0.08}} & \makecell[tc]{96.60\,{\scriptsize$\pm$\,0.07} \\ \negred{$-$0.09}} \\
        CS & \makecell[tc]{91.54 \\[-0.7ex] {\scriptsize$\pm$\,0.20}} & \makecell[tc]{95.88 \\[-0.7ex] {\scriptsize$\pm$\,0.03}} & \makecell[tc]{96.29 \\[-0.7ex] {\scriptsize$\pm$\,0.12}} & \makecell[tc]{96.17 \\[-0.7ex] {\scriptsize$\pm$\,0.02}} & & \makecell[tc]{91.77\,{\scriptsize$\pm$\,0.38} \\ \posgreen{$+$0.23}} & \makecell[tc]{95.89\,{\scriptsize$\pm$\,0.06} \\ \posgreen{$+$0.01}} & \makecell[tc]{96.24\,{\scriptsize$\pm$\,0.10} \\ \negred{$-$0.05}} & \makecell[tc]{96.17\,{\scriptsize$\pm$\,0.02} \\ $+$0.00} \\
        Physics & \makecell[tc]{95.97 \\[-0.7ex] {\scriptsize$\pm$\,0.07}} & \makecell[tc]{97.38 \\[-0.7ex] {\scriptsize$\pm$\,0.06}} & \makecell[tc]{97.25 \\[-0.7ex] {\scriptsize$\pm$\,0.08}} & \makecell[tc]{97.26 \\[-0.7ex] {\scriptsize$\pm$\,0.03}} & & \makecell[tc]{95.98\,{\scriptsize$\pm$\,0.08} \\ \posgreen{$+$0.01}} & \makecell[tc]{97.44\,{\scriptsize$\pm$\,0.14} \\ \posgreen{$+$0.06}} & \makecell[tc]{97.23\,{\scriptsize$\pm{<}0.01$} \\ \negred{$-$0.02}} & \makecell[tc]{97.38\,{\scriptsize$\pm$\,0.04} \\ \textbf{\posgreen{$+$0.12}}} \\
        WikiCS & \makecell[tc]{70.96 \\[-0.7ex] {\scriptsize$\pm$\,1.00}} & \makecell[tc]{79.97 \\[-0.7ex] {\scriptsize$\pm$\,0.43}} & \makecell[tc]{80.71 \\[-0.7ex] {\scriptsize$\pm$\,0.19}} & \makecell[tc]{80.92 \\[-0.7ex] {\scriptsize$\pm$\,0.58}} & & \makecell[tc]{72.48\,{\scriptsize$\pm$\,0.83} \\ \textbf{\posgreen{$+$1.52}}} & \makecell[tc]{80.31\,{\scriptsize$\pm$\,0.44} \\ \posgreen{$+$0.34}} & \makecell[tc]{81.10\,{\scriptsize$\pm$\,0.32} \\ \textbf{\posgreen{$+$0.39}}} & \makecell[tc]{81.78\,{\scriptsize$\pm$\,0.22} \\ \textbf{\posgreen{$+$0.86}}} \\
        \midrule
        Squirrel & \makecell[tc]{39.30 \\[-0.7ex] {\scriptsize$\pm$\,0.79}} & \makecell[tc]{43.75 \\[-0.7ex] {\scriptsize$\pm$\,1.91}} & \makecell[tc]{40.48 \\[-0.7ex] {\scriptsize$\pm$\,2.90}} & \makecell[tc]{41.51 \\[-0.7ex] {\scriptsize$\pm$\,2.34}} & & \makecell[tc]{39.39\,{\scriptsize$\pm$\,0.91} \\ \posgreen{$+$0.09}} & \makecell[tc]{44.57\,{\scriptsize$\pm$\,2.04} \\ \posgreen{$+$0.82}} & \makecell[tc]{41.32\,{\scriptsize$\pm$\,2.35} \\ \posgreen{$+$0.84}} & \makecell[tc]{40.36\,{\scriptsize$\pm$\,1.66} \\ \negred{$-$1.15}} \\
        Chameleon & \makecell[tc]{43.86 \\[-0.7ex] {\scriptsize$\pm$\,5.23}} & \makecell[tc]{45.30 \\[-0.7ex] {\scriptsize$\pm$\,2.30}} & \makecell[tc]{44.32 \\[-0.7ex] {\scriptsize$\pm$\,4.55}} & \makecell[tc]{43.07 \\[-0.7ex] {\scriptsize$\pm$\,5.25}} & & \makecell[tc]{43.86\,{\scriptsize$\pm$\,5.23} \\ $+$0.00} & \makecell[tc]{45.27\,{\scriptsize$\pm$\,4.74} \\ \negred{$-$0.03}} & \makecell[tc]{43.32\,{\scriptsize$\pm$\,4.67} \\ \negred{$-$1.00}} & \makecell[tc]{44.52\,{\scriptsize$\pm$\,3.69} \\ \posgreen{$+$1.45}} \\
        Amazon-Rat. & \makecell[tc]{48.85 \\[-0.7ex] {\scriptsize$\pm$\,0.55}} & \makecell[tc]{53.64 \\[-0.7ex] {\scriptsize$\pm$\,0.54}} & \makecell[tc]{55.18 \\[-0.7ex] {\scriptsize$\pm$\,0.93}} & \makecell[tc]{55.09 \\[-0.7ex] {\scriptsize$\pm$\,0.19}} & & \makecell[tc]{49.50\,{\scriptsize$\pm$\,0.27} \\ \textbf{\posgreen{$+$0.65}}} & \makecell[tc]{54.06\,{\scriptsize$\pm$\,0.58} \\ \posgreen{$+$0.42}} & \makecell[tc]{56.72\,{\scriptsize$\pm$\,0.35} \\ \textbf{\posgreen{$+$1.54}}} & \makecell[tc]{55.73\,{\scriptsize$\pm$\,0.31} \\ \textbf{\posgreen{$+$0.64}}} \\
        Roman-Emp. & \makecell[tc]{66.10 \\[-0.7ex] {\scriptsize$\pm$\,0.44}} & \makecell[tc]{91.15 \\[-0.7ex] {\scriptsize$\pm$\,0.20}} & \makecell[tc]{90.50 \\[-0.7ex] {\scriptsize$\pm$\,0.21}} & \makecell[tc]{90.49 \\[-0.7ex] {\scriptsize$\pm$\,0.22}} & & \makecell[tc]{66.12\,{\scriptsize$\pm$\,0.29} \\ \posgreen{$+$0.02}} & \makecell[tc]{91.66\,{\scriptsize$\pm$\,0.20} \\ \textbf{\posgreen{$+$0.51}}} & \makecell[tc]{91.27\,{\scriptsize$\pm$\,0.36} \\ \textbf{\posgreen{$+$0.77}}} & \makecell[tc]{90.93\,{\scriptsize$\pm$\,0.20} \\ \textbf{\posgreen{$+$0.44}}} \\
        Minesweeper & \makecell[tc]{51.06 \\[-0.7ex] {\scriptsize$\pm$\,1.76}} & \makecell[tc]{97.26 \\[-0.7ex] {\scriptsize$\pm$\,0.22}} & \makecell[tc]{97.09 \\[-0.7ex] {\scriptsize$\pm$\,1.00}} & \makecell[tc]{97.86 \\[-0.7ex] {\scriptsize$\pm$\,0.37}} & & \makecell[tc]{50.97\,{\scriptsize$\pm$\,1.56} \\ \negred{$-$0.09}} & \makecell[tc]{97.80\,{\scriptsize$\pm$\,0.20} \\ \textbf{\posgreen{$+$0.54}}} & \makecell[tc]{97.33\,{\scriptsize$\pm$\,0.94} \\ \posgreen{$+$0.24}} & \makecell[tc]{97.86\,{\scriptsize$\pm$\,0.37} \\ $+$0.00} \\
        \bottomrule
        \end{tabular}
    }
\end{table}

Table~\ref{tab:results} reports test performance for each dataset--model pair, comparing the supervised baseline with the corresponding TS-augmented model across 13 different datasets. Our key takeaways are as follows:
\begin{enumerate}[
    label=(A\arabic*),
    leftmargin=2.8em,
    labelsep=0.8em,
    itemsep=0.4em,
    align=left
]
    \item \textbf{TS improves standard GNN training.} Across GCN, GraphSAGE, and GAT, adding TS generally matches or improves the corresponding supervised baseline. This supports the view that unlabeled-node predictions contain recoverable training signal, as formalized by Definition~\ref{def:transductive-regularization}, and that uncertainty provides an effective way to recover it.

    \item \textbf{TS does not replace message passing.} TS also improves MLPs on several datasets, showing that the sharpening signal is not specific to GNN architectures. However, MLP+TS remains below the corresponding GNN performance, indicating that TS does not replace the benefit of message passing.
\end{enumerate}

\begin{table}[H]
    \centering
    \caption{Test performance of a method using unlabeled nodes, on identical backbones, splits, and paired seeds. Self-training (ST) is validation-tuned per cell; TS uses the fixed $\lambda=0.25$. Means $\pm$ std over 5 seeds; best per column within each block in \textbf{bold}.}
    \label{tab:unlabeled_baselines}
    \fittotextwidth{%
        \setlength{\tabcolsep}{3pt}%
        \begin{tabular}{l ccccc}
            \toprule
             & Cora & CiteSeer & PubMed & Roman-Emp. & Amazon-Rat. \\
            \midrule
            GCN & 84.74 {\scriptsize$\pm$\,0.56} & 72.68 {\scriptsize$\pm$\,0.45} & 80.70 {\scriptsize$\pm$\,0.96} & 91.15 {\scriptsize$\pm$\,0.29} & 53.86 {\scriptsize$\pm$\,0.22} \\
            \;+ ST & \textbf{86.02} {\scriptsize$\pm$\,0.33} & 74.36 {\scriptsize$\pm$\,0.58} & \textbf{81.46} {\scriptsize$\pm$\,0.43} & \textbf{91.68} {\scriptsize$\pm$\,0.20} & \textbf{54.04} {\scriptsize$\pm$\,0.27} \\
            \;+ TS & 85.14 {\scriptsize$\pm$\,0.34} & \textbf{75.10} {\scriptsize$\pm$\,0.19} & 80.90 {\scriptsize$\pm$\,0.28} & 91.66 {\scriptsize$\pm$\,0.14} & 53.97 {\scriptsize$\pm$\,0.70} \\
            \midrule
            SAGE & 83.58 {\scriptsize$\pm$\,0.58} & 69.62 {\scriptsize$\pm$\,0.60} & 77.94 {\scriptsize$\pm$\,1.30} & 90.60 {\scriptsize$\pm$\,0.14} & 55.02 {\scriptsize$\pm$\,0.65} \\
            \;+ ST & \textbf{85.10} {\scriptsize$\pm$\,1.08} & \textbf{73.82} {\scriptsize$\pm$\,0.83} & \textbf{79.76} {\scriptsize$\pm$\,1.09} & 91.10 {\scriptsize$\pm$\,0.17} & 55.42 {\scriptsize$\pm$\,0.39} \\
            \;+ TS & 84.92 {\scriptsize$\pm$\,0.96} & 72.00 {\scriptsize$\pm$\,1.36} & 79.00 {\scriptsize$\pm$\,0.53} & \textbf{91.12} {\scriptsize$\pm$\,0.23} & \textbf{55.91} {\scriptsize$\pm$\,0.64} \\
            \midrule
            GAT & 82.40 {\scriptsize$\pm$\,1.01} & 71.86 {\scriptsize$\pm$\,0.33} & \textbf{79.76} {\scriptsize$\pm$\,1.21} & 90.23 {\scriptsize$\pm$\,0.26} & 55.42 {\scriptsize$\pm$\,0.13} \\
            \;+ ST & \textbf{84.62} {\scriptsize$\pm$\,1.11} & 74.30 {\scriptsize$\pm$\,0.81} & 78.98 {\scriptsize$\pm$\,2.12} & 90.90 {\scriptsize$\pm$\,0.07} & \textbf{55.96} {\scriptsize$\pm$\,0.36} \\
            \;+ TS & 83.70 {\scriptsize$\pm$\,1.19} & \textbf{74.80} {\scriptsize$\pm$\,0.81} & 79.22 {\scriptsize$\pm$\,2.07} & \textbf{90.96} {\scriptsize$\pm$\,0.43} & 55.69 {\scriptsize$\pm$\,0.45} \\
            \bottomrule
        \end{tabular}
    }
\end{table}

\begin{table}[H]
    \centering
    \caption{Efficiency comparison between self-training (ST) and TS, measured as test accuracy per unit training time, $1000\times\mathrm{accuracy}/\mathrm{seconds}$. Higher values indicate greater performance per unit computation. Training time includes ST's eight configurations of validation tuning.}
    \label{tab:efficiency_comparison}
    \fittotextwidth{%
        \setlength{\tabcolsep}{4pt}%
        \begin{tabular}{l cc cc cc cc cc}
            \toprule
            & \multicolumn{2}{c}{Cora}
            & \multicolumn{2}{c}{CiteSeer}
            & \multicolumn{2}{c}{PubMed}
            & \multicolumn{2}{c}{Roman-Emp.}
            & \multicolumn{2}{c}{Amazon-Rat.} \\
            \cmidrule(lr){2-3}
            \cmidrule(lr){4-5}
            \cmidrule(lr){6-7}
            \cmidrule(lr){8-9}
            \cmidrule(lr){10-11}
            & ST & TS & ST & TS & ST & TS & ST & TS & ST & TS \\
            \midrule

            GCN
            & 432.26 & \textbf{2580.00}
            & 359.92 & \textbf{2823.31}
            & 387.90 & \textbf{3502.16}
            & 20.94 & \textbf{123.81}
            & 15.39 & \textbf{124.04} \\

            SAGE
            & 340.81 & \textbf{2645.48}
            & 291.32 & \textbf{2250.00}
            & 126.44 & \textbf{1130.19}
            & 35.33 & \textbf{283.33}
            & 7.53 & \textbf{59.26} \\

            GAT
            & 329.65 & \textbf{2827.70}
            & 173.96 & \textbf{1427.48}
            & 130.20 & \textbf{1148.12}
            & 12.45 & \textbf{113.84}
            & 16.89 & \textbf{90.00} \\

            \bottomrule
        \end{tabular}
    }
\end{table}

Table~\ref{tab:unlabeled_baselines} reports test performance for each dataset-model pair, comparing TS with Self-training (ST), a standard class of methods that leverages unlabeled nodes during training. Table~\ref{tab:efficiency_comparison} reports the performance-compute efficiency of TS and ST across datasets and backbone architectures.
Even excluding the tuning budget, a single pre-configured ST deployment costs $2$--$3\times$ a TS run, since each pseudo-labeling round retrains the backbone, whereas TS modifies only the loss of a single training run.

\begin{enumerate}[
label=(A3),
leftmargin=2.8em,
labelsep=0.8em,
itemsep=0.4em,
align=left
]
\item \textbf{TS is competitive with existing methods that exploit unlabeled nodes.} As shown in Table~\ref{tab:unlabeled_baselines}, with a single fixed $\lambda=0.25$, TS often approaches the performance of tuned self-training across datasets and backbones. Moreover, as shown in Table~\ref{tab:efficiency_comparison}, TS achieves substantially higher accuracy per unit of training time, showing that it is much more computationally efficient.
\end{enumerate}

Overall, these results demonstrate that TS provides a simple and broadly applicable improvement to transductive node classification.

\subsection{On the Effect of the Sharpening Coefficient}
\label{sec:lambda-effect}

To compare $\lambda$ values across datasets with different baseline performance levels and variances, we report improvements using Glass's $\Delta$,
\(
    \Delta_{\mathrm{Glass}}(\lambda)
    =
    \frac{\mathrm{Acc}_{\lambda} - \mathrm{Acc}_{0}}
    {\sigma_{0}},
\)
where $\mathrm{Acc}_{0}$ and $\sigma_{0}$ are the mean and standard deviation of the corresponding $\lambda=0$ supervised baseline. This normalization measures each gain in units of the baseline variability and is preferable to raw metric differences in this case because the datasets differ substantially in both difficulty and noise.

Figure~\ref{fig:lambda_glass_overlay} aggregates the Glass-normalized gains over the 13 datasets for each GNN backbone. The median curves remain close to or above zero for small positive values of $\lambda$, with the most stable region lying roughly between $\lambda=0$ and $\lambda=0.5$. Beyond this range, the curves gradually deteriorate, and large values become harmful more often.

\begin{figure}[htbp]
    \centering
    \includegraphics[width=0.85\textwidth]{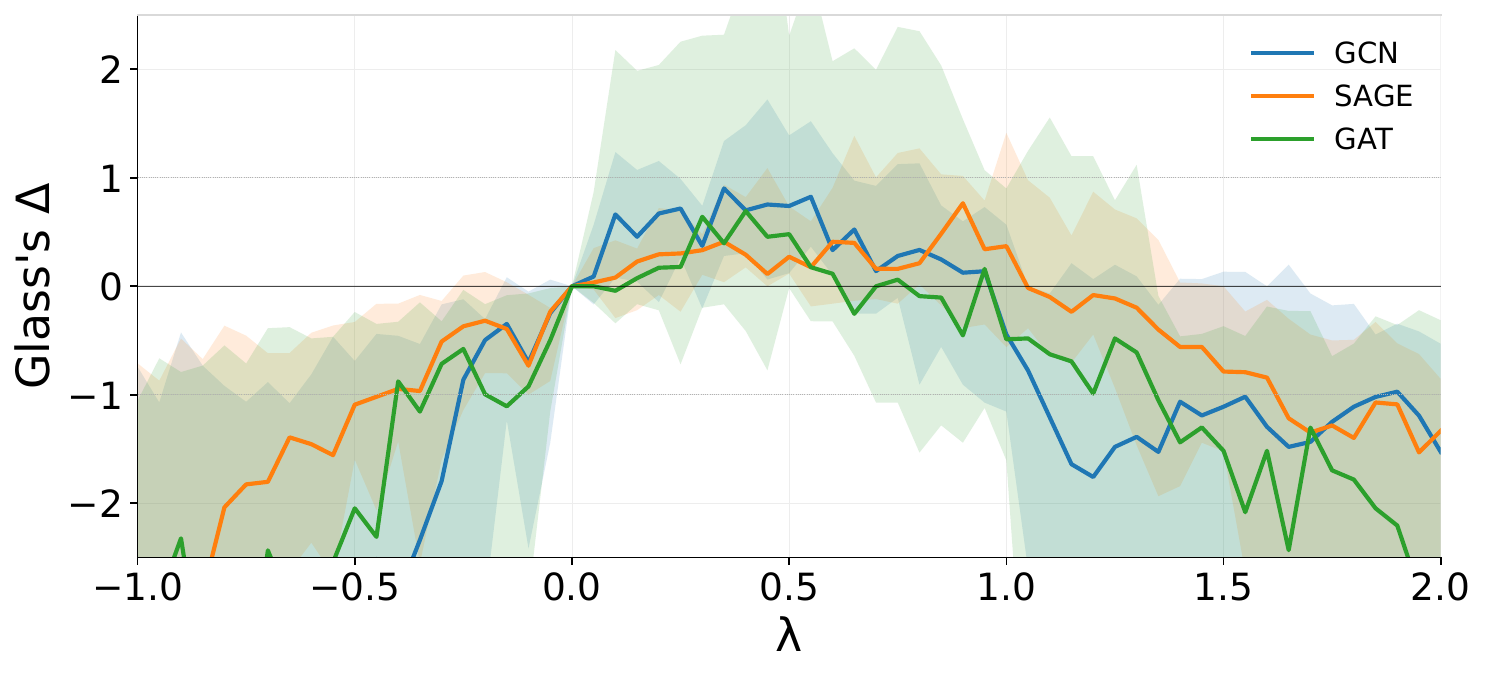}
    \caption{Glass's $\Delta$ on test performance vs.\ $\lambda$, where each curve aggregates one GNN backbone over its 13 datasets. Solid line denotes the median and the shaded band denotes the interquartile range. On the benchmarks studied, $\lambda \in [0,0.5]$ is the most reliable region: gains concentrate there and regressions within it are rare and small. Additional per-\(\lambda\) visualizations, including the full distribution of improvements and regressions and the corresponding per-dataset performance curves, are provided in Appendix~\ref{app:further-visualizations}.}
    \label{fig:lambda_glass_overlay}
\end{figure}

\begin{enumerate}[
    label=(A4),
    leftmargin=2.8em,
    labelsep=0.8em,
    itemsep=0.4em,
    align=left
]
    \item \textbf{Effect of $\lambda$:}
    Across datasets and backbones, Figure~\ref{fig:lambda_glass_overlay} shows a broadly consistent relationship between $\lambda$ and accuracy: performance typically improves for small positive values, reaches a plateau or local maximum at moderate sharpening strength, and then degrades when $\lambda$ becomes too large. This behavior matches the intuition behind TS: a mild sharpening signal can help the model exploit reliable unlabeled-node predictions, whereas excessive sharpening may force the model to commit too strongly to incorrect predictions.
\end{enumerate}

The consistency of the pattern observed in Figure~\ref{fig:lambda_glass_overlay} suggests that the useful range of $\lambda$ is not entirely dataset-specific. We therefore study whether a single coefficient can work reasonably well across datasets and architectures. Based on Figure~\ref{fig:lambda_glass_overlay}, we choose the midpoint of the stable region, $\lambda=0.25$, as a simple universal setting.

\begin{table}[htbp]
    \centering
    \caption{Per-dataset results for TS at $\lambda{=}0.25$ (universal). Each cell shows test performance $\pm$ std on top and the $\Delta$ vs.\ the matching baseline below. \posgreen{Green} when positive, \negred{red} when negative; \textbf{bold} when $|\Delta| > \sigma$. We observe that TS, even without tuning $\lambda$ as a hyperparameter, improves the performance of classic baselines.}
    \label{tab:lambda025}
    \fittotextwidth{%
        \setlength{\tabcolsep}{3pt}%
        \begin{tabular}{l ccccccccccccc}
        \toprule
         & Cora & CiteSeer & PubMed & Computer & Photo & CS & Physics & WikiCS & Squirrel & Chameleon & Amazon-Rat. & Roman-Emp. & Minesweeper \\
         \cmidrule(lr){2-9} \cmidrule(lr){10-14}
        GCN & \makecell[tc]{85.16 \\[-0.7ex] {\scriptsize$\pm$\,0.36} \\ \posgreen{$+$0.62}} & \makecell[tc]{75.10 \\[-0.7ex] {\scriptsize$\pm$\,0.19} \\ \textbf{\posgreen{$+$2.42}}} & \makecell[tc]{80.88 \\[-0.7ex] {\scriptsize$\pm$\,0.31} \\ \posgreen{$+$0.18}} & \makecell[tc]{93.81 \\[-0.7ex] {\scriptsize$\pm$\,0.22} \\ \textbf{\negred{$-$0.31}}} & \makecell[tc]{96.32 \\[-0.7ex] {\scriptsize$\pm$\,0.08} \\ \textbf{\posgreen{$+$0.42}}} & \makecell[tc]{95.90 \\[-0.7ex] {\scriptsize$\pm$\,0.03} \\ \posgreen{$+$0.02}} & \makecell[tc]{97.44 \\[-0.7ex] {\scriptsize$\pm$\,0.14} \\ \posgreen{$+$0.06}} & \makecell[tc]{80.17 \\[-0.7ex] {\scriptsize$\pm$\,0.48} \\ \posgreen{$+$0.20}} & \makecell[tc]{44.25 \\[-0.7ex] {\scriptsize$\pm$\,2.04} \\ \posgreen{$+$0.50}} & \makecell[tc]{45.27 \\[-0.7ex] {\scriptsize$\pm$\,4.74} \\ \negred{$-$0.03}} & \makecell[tc]{54.11 \\[-0.7ex] {\scriptsize$\pm$\,0.34} \\ \posgreen{$+$0.47}} & \makecell[tc]{91.57 \\[-0.7ex] {\scriptsize$\pm$\,0.20} \\ \textbf{\posgreen{$+$0.42}}} & \makecell[tc]{97.44 \\[-0.7ex] {\scriptsize$\pm$\,0.17} \\ \posgreen{$+$0.18}} \\
        SAGE & \makecell[tc]{84.94 \\[-0.7ex] {\scriptsize$\pm$\,0.89} \\ \textbf{\posgreen{$+$1.34}}} & \makecell[tc]{72.00 \\[-0.7ex] {\scriptsize$\pm$\,1.36} \\ \textbf{\posgreen{$+$2.40}}} & \makecell[tc]{78.40 \\[-0.7ex] {\scriptsize$\pm$\,1.12} \\ \posgreen{$+$0.54}} & \makecell[tc]{93.51 \\[-0.7ex] {\scriptsize$\pm$\,0.06} \\ \posgreen{$+$0.26}} & \makecell[tc]{96.41 \\[-0.7ex] {\scriptsize$\pm$\,0.41} \\ \negred{$-$0.02}} & \makecell[tc]{96.20 \\[-0.7ex] {\scriptsize$\pm$\,0.06} \\ \negred{$-$0.09}} & \makecell[tc]{97.23 \\[-0.7ex] {\scriptsize$\pm{<}0.01$} \\ \negred{$-$0.02}} & \makecell[tc]{80.77 \\[-0.7ex] {\scriptsize$\pm$\,0.14} \\ \posgreen{$+$0.06}} & \makecell[tc]{39.77 \\[-0.7ex] {\scriptsize$\pm$\,2.13} \\ \negred{$-$0.71}} & \makecell[tc]{42.93 \\[-0.7ex] {\scriptsize$\pm$\,5.67} \\ \negred{$-$1.39}} & \makecell[tc]{55.69 \\[-0.7ex] {\scriptsize$\pm$\,0.23} \\ \posgreen{$+$0.51}} & \makecell[tc]{91.21 \\[-0.7ex] {\scriptsize$\pm$\,0.25} \\ \textbf{\posgreen{$+$0.71}}} & \makecell[tc]{96.96 \\[-0.7ex] {\scriptsize$\pm$\,0.47} \\ \negred{$-$0.13}} \\
        GAT & \makecell[tc]{83.70 \\[-0.7ex] {\scriptsize$\pm$\,1.19} \\ \posgreen{$+$1.30}} & \makecell[tc]{74.80 \\[-0.7ex] {\scriptsize$\pm$\,0.81} \\ \textbf{\posgreen{$+$2.90}}} & \makecell[tc]{79.80 \\[-0.7ex] {\scriptsize$\pm$\,0.93} \\ \posgreen{$+$0.04}} & \makecell[tc]{93.82 \\[-0.7ex] {\scriptsize$\pm$\,0.13} \\ \negred{$-$0.16}} & \makecell[tc]{96.47 \\[-0.7ex] {\scriptsize$\pm{<}0.01$} \\ \textbf{\negred{$-$0.22}}} & \makecell[tc]{96.19 \\[-0.7ex] {\scriptsize$\pm$\,0.11} \\ \posgreen{$+$0.02}} & \makecell[tc]{97.32 \\[-0.7ex] {\scriptsize$\pm$\,0.06} \\ \posgreen{$+$0.06}} & \makecell[tc]{81.03 \\[-0.7ex] {\scriptsize$\pm$\,0.92} \\ \posgreen{$+$0.11}} & \makecell[tc]{39.53 \\[-0.7ex] {\scriptsize$\pm$\,2.12} \\ \negred{$-$1.98}} & \makecell[tc]{43.87 \\[-0.7ex] {\scriptsize$\pm$\,5.20} \\ \posgreen{$+$0.80}} & \makecell[tc]{55.59 \\[-0.7ex] {\scriptsize$\pm$\,0.27} \\ \textbf{\posgreen{$+$0.50}}} & \makecell[tc]{90.99 \\[-0.7ex] {\scriptsize$\pm$\,0.11} \\ \textbf{\posgreen{$+$0.50}}} & \makecell[tc]{97.10 \\[-0.7ex] {\scriptsize$\pm$\,0.69} \\ \negred{$-$0.76}} \\
        \bottomrule
        \end{tabular}
    }
\end{table}

\begin{enumerate}[
    label=(A5),
    leftmargin=2.8em,
    labelsep=0.8em,
    itemsep=0.4em,
    align=left
]
    \item \textbf{Universal $\lambda$:}
    Table~\ref{tab:lambda025} shows that despite removing per-dataset tuning, the universal setting preserves much of the benefit of TS: it improves many model-dataset pairs and rarely causes large regressions. This suggests that TS is not overly sensitive to precise coefficient selection, provided that $\lambda$ is chosen in a conservative positive range.
\end{enumerate}

\subsection{Ablation Studies}

We evaluate the impact of key design choices in TS through a series of ablations. Further ablations are provided in Appendix~\ref{sec:further-ablations}.

\paragraph{Removing the labeled-node correction.} At $\lambda = 0.25$, we evaluate a variant that applies the entropy-minimization term to unlabeled nodes while removing the entropy-maximization term on labeled training nodes.  Table~\ref{tab:l025_lt0_changes} shows that removing this correction generally does not improve performance.

\begin{table}[H]
    \centering
    \caption{Test-accuracy gain of dropping the labelled-side entropy-max term ($\lambda_L{=}0$) over the symmetric default ($\lambda_L{=}{-}\lambda_U$).  \posgreen{Green} when positive, \negred{red} when negative; \textbf{bold} when $|\Delta| > \sigma$.}
    \label{tab:l025_lt0_changes}
    \fittotextwidth{%
        \setlength{\tabcolsep}{3pt}%
        \begin{tabular}{l ccccccccccccc}
        \toprule
         & Cora & CiteSeer & PubMed & Computer & Photo & CS & Physics & WikiCS & Squirrel & Chameleon & Amazon-Rat. & Roman-Emp. & Minesweeper \\
         \cmidrule(lr){2-9} \cmidrule(lr){10-14}
        GCN & \makecell[tc]{\negred{$-$0.28} \\[-0.7ex] {\scriptsize$\pm$\,0.66}} & \makecell[tc]{\negred{$-$0.02} \\[-0.7ex] {\scriptsize$\pm$\,0.26}} & \makecell[tc]{\negred{$-$0.12} \\[-0.7ex] {\scriptsize$\pm$\,0.54}} & \makecell[tc]{\posgreen{$+$0.07} \\[-0.7ex] {\scriptsize$\pm$\,0.31}} & \makecell[tc]{\negred{$-$0.13} \\[-0.7ex] {\scriptsize$\pm$\,0.20}} & \makecell[tc]{\posgreen{$+$0.06} \\[-0.7ex] {\scriptsize$\pm$\,0.10}} & \makecell[tc]{\posgreen{$+$0.03} \\[-0.7ex] {\scriptsize$\pm$\,0.16}} & \makecell[tc]{\negred{$-$0.10} \\[-0.7ex] {\scriptsize$\pm$\,0.62}} & \makecell[tc]{\negred{$-$0.54} \\[-0.7ex] {\scriptsize$\pm$\,3.16}} & \makecell[tc]{\negred{$-$1.08} \\[-0.7ex] {\scriptsize$\pm$\,6.54}} & \makecell[tc]{\posgreen{$+$0.01} \\[-0.7ex] {\scriptsize$\pm$\,0.58}} & \makecell[tc]{\negred{$-$0.05} \\[-0.7ex] {\scriptsize$\pm$\,0.24}} & \makecell[tc]{\negred{$-$0.02} \\[-0.7ex] {\scriptsize$\pm$\,0.29}} \\
        SAGE & \makecell[tc]{\negred{$-$0.12} \\[-0.7ex] {\scriptsize$\pm$\,1.39}} & \makecell[tc]{\textbf{\posgreen{$+$1.58}} \\[-0.7ex] {\scriptsize$\pm$\,1.43}} & \makecell[tc]{\negred{$-$0.12} \\[-0.7ex] {\scriptsize$\pm$\,1.49}} & \makecell[tc]{\posgreen{$+$0.10} \\[-0.7ex] {\scriptsize$\pm$\,0.18}} & \makecell[tc]{\posgreen{$+$0.15} \\[-0.7ex] {\scriptsize$\pm$\,0.54}} & \makecell[tc]{\posgreen{$+$0.05} \\[-0.7ex] {\scriptsize$\pm$\,0.08}} & \makecell[tc]{0 \\[-0.7ex] {\scriptsize$\pm$\,0.01}} & \makecell[tc]{\negred{$-$0.31} \\[-0.7ex] {\scriptsize$\pm$\,0.49}} & \makecell[tc]{\posgreen{$+$0.79} \\[-0.7ex] {\scriptsize$\pm$\,3.32}} & \makecell[tc]{\negred{$-$0.19} \\[-0.7ex] {\scriptsize$\pm$\,7.46}} & \makecell[tc]{\negred{$-$0.33} \\[-0.7ex] {\scriptsize$\pm$\,0.68}} & \makecell[tc]{\negred{$-$0.17} \\[-0.7ex] {\scriptsize$\pm$\,0.47}} & \makecell[tc]{\posgreen{$+$0.22} \\[-0.7ex] {\scriptsize$\pm$\,0.54}} \\
        GAT & \makecell[tc]{\negred{$-$0.36} \\[-0.7ex] {\scriptsize$\pm$\,1.41}} & \makecell[tc]{\negred{$-$0.26} \\[-0.7ex] {\scriptsize$\pm$\,1.04}} & \makecell[tc]{\negred{$-$0.40} \\[-0.7ex] {\scriptsize$\pm$\,1.24}} & \makecell[tc]{\posgreen{$+$0.16} \\[-0.7ex] {\scriptsize$\pm$\,0.19}} & \makecell[tc]{\negred{$-$0.04} \\[-0.7ex] {\scriptsize$\pm$\,0.14}} & \makecell[tc]{\posgreen{$+$0.05} \\[-0.7ex] {\scriptsize$\pm$\,0.15}} & \makecell[tc]{0 \\[-0.7ex] {\scriptsize$\pm$\,0.09}} & \makecell[tc]{\posgreen{$+$0.19} \\[-0.7ex] {\scriptsize$\pm$\,1.22}} & \makecell[tc]{\negred{$-$0.02} \\[-0.7ex] {\scriptsize$\pm$\,3.17}} & \makecell[tc]{\negred{$-$0.02} \\[-0.7ex] {\scriptsize$\pm$\,6.80}} & \makecell[tc]{\posgreen{$+$0.26} \\[-0.7ex] {\scriptsize$\pm$\,0.43}} & \makecell[tc]{\negred{$-$0.08} \\[-0.7ex] {\scriptsize$\pm$\,0.44}} & \makecell[tc]{\posgreen{$+$0.17} \\[-0.7ex] {\scriptsize$\pm$\,0.75}} \\
        \bottomrule
        \end{tabular}
    }
\end{table}

\paragraph{Symmetry of $\lambda$.}
We test whether the symmetric formulation ($\lambda_U = -\lambda_L$) is essential by introducing an offset \(\{-0.10, -0.05, +0.05, +0.10\}\) while keeping $\lambda_U - \lambda_L$ fixed.
Table~\ref{tab:offset_changes} shows that asymmetric configurations do not yield systematic improvements, and often slightly degrade performance, suggesting that the balance between sharpening unlabeled nodes and regularizing labeled ones is important for stable gains.
\begin{table}[H]
    \centering
    \caption{Best test-accuracy gain from the non-symmetric offset, relative to the symmetric default $\lambda_U = -\lambda_L$.  \posgreen{Green} when positive, \negred{red} when negative; \textbf{bold} when $|\Delta| > \sigma$. This table shows the symmetric approach outperforms adding non-symmetry to the Generic Transductive Sharpening Objective (Definition~\ref{def:generic-TS}).}
    \label{tab:offset_changes}
    \fittotextwidth{%
        \setlength{\tabcolsep}{3pt}%
        \begin{tabular}{l ccccccccccccc}
        \toprule
         & Cora & CiteSeer & PubMed & Computer & Photo & CS & Physics & WikiCS & Squirrel & Chameleon & Amazon-Rat. & Roman-Emp. & Minesweeper \\
         \cmidrule(lr){2-9} \cmidrule(lr){10-14}
        GCN & \makecell[tc]{\negred{$-$0.24} \\[-0.7ex] {\scriptsize$\pm$\,0.49}} & \makecell[tc]{\negred{$-$0.10} \\[-0.7ex] {\scriptsize$\pm$\,0.67}} & \makecell[tc]{\negred{$-$0.20} \\[-0.7ex] {\scriptsize$\pm$\,0.94}} & 0 & \makecell[tc]{\negred{$-$0.26} \\[-0.7ex] {\scriptsize$\pm$\,0.36}} & \makecell[tc]{\negred{$-$0.07} \\[-0.7ex] {\scriptsize$\pm$\,0.11}} & \makecell[tc]{\negred{$-$0.05} \\[-0.7ex] {\scriptsize$\pm$\,0.18}} & \makecell[tc]{\negred{$-$0.25} \\[-0.7ex] {\scriptsize$\pm$\,0.71}} & \makecell[tc]{\negred{$-$0.57} \\[-0.7ex] {\scriptsize$\pm$\,2.82}} & 0 & \makecell[tc]{\negred{$-$0.03} \\[-0.7ex] {\scriptsize$\pm$\,0.34}} & \makecell[tc]{\negred{$-$0.17} \\[-0.7ex] {\scriptsize$\pm$\,0.20}} & 0 \\
        SAGE & \makecell[tc]{0 \\[-0.7ex] {\scriptsize$\pm$\,1.35}} & 0 & \makecell[tc]{\negred{$-$0.84} \\[-0.7ex] {\scriptsize$\pm$\,0.85}} & 0 & 0 & \makecell[tc]{\negred{$-$0.05} \\[-0.7ex] {\scriptsize$\pm$\,0.20}} & \makecell[tc]{0 \\[-0.7ex] {\scriptsize$\pm$\,0.06}} & \makecell[tc]{\negred{$-$0.12} \\[-0.7ex] {\scriptsize$\pm$\,0.43}} & \makecell[tc]{\negred{$-$1.29} \\[-0.7ex] {\scriptsize$\pm$\,3.10}} & 0 & \makecell[tc]{\posgreen{$+$0.14} \\[-0.7ex] {\scriptsize$\pm$\,0.54}} & \makecell[tc]{\negred{$-$0.07} \\[-0.7ex] {\scriptsize$\pm$\,0.21}} & 0 \\
        GAT & 0 & 0 & \makecell[tc]{\negred{$-$0.26} \\[-0.7ex] {\scriptsize$\pm$\,1.19}} & 0 & \makecell[tc]{\textbf{\negred{$-$0.31}} \\[-0.7ex] {\scriptsize$\pm$\,0.08}} & \makecell[tc]{\textbf{\negred{$-$0.08}} \\[-0.7ex] {\scriptsize$\pm$\,0.05}} & \makecell[tc]{\negred{$-$0.07} \\[-0.7ex] {\scriptsize$\pm$\,0.13}} & \makecell[tc]{\negred{$-$0.23} \\[-0.7ex] {\scriptsize$\pm$\,0.24}} & \makecell[tc]{\negred{$-$0.69} \\[-0.7ex] {\scriptsize$\pm$\,2.73}} & \makecell[tc]{\negred{$-$1.07} \\[-0.7ex] {\scriptsize$\pm$\,6.04}} & \makecell[tc]{\negred{$-$0.28} \\[-0.7ex] {\scriptsize$\pm$\,0.35}} & \makecell[tc]{\posgreen{$+$0.08} \\[-0.7ex] {\scriptsize$\pm$\,0.43}} & 0 \\
        \bottomrule
        \end{tabular}
    }
\end{table}

\paragraph{Choice of entropy.}
We compare Tsallis entropy ($q=2$) with Shannon entropy ($q=1$) at a fixed $\lambda=0.25$. While average accuracy remains similar, Shannon entropy leads to higher variance and instability in some cases (see GAT$+$Roman-Empire cell). We therefore use Tsallis entropy as a simple, practical default.
\begin{table}[H]
    \centering
    \caption{Test-accuracy gain of Shannon entropy ($q{=}1$) over Gini ($q{=}2$) at $\lambda{=}0.25$.  \posgreen{Green} when positive, \negred{red} when negative; \textbf{bold} when $|\Delta| > \sigma$. These results show the choice of entropy generally does not affect the reported accuracy. Thus motivating the choice of Tsallis entropy for its simple quadratic, log-free form.}
    \label{tab:shannon_changes}
    \fittotextwidth{%
        \setlength{\tabcolsep}{3pt}%
        \begin{tabular}{l ccccccccccccc}
        \toprule
         & Cora & CiteSeer & PubMed & Computer & Photo & CS & Physics & WikiCS & Squirrel & Chameleon & Amazon-Rat. & Roman-Emp. & Minesweeper \\
         \cmidrule(lr){2-9} \cmidrule(lr){10-14}
        GCN & \makecell[tc]{\negred{$-$0.14} \\[-0.7ex] {\scriptsize$\pm$\,0.77}} & \makecell[tc]{\negred{$-$0.12} \\[-0.7ex] {\scriptsize$\pm$\,0.38}} & \makecell[tc]{\negred{$-$0.16} \\[-0.7ex] {\scriptsize$\pm$\,0.57}} & \makecell[tc]{\posgreen{$+$0.21} \\[-0.7ex] {\scriptsize$\pm$\,0.24}} & \makecell[tc]{\textbf{\negred{$-$0.15}} \\[-0.7ex] {\scriptsize$\pm$\,0.13}} & \makecell[tc]{0 \\[-0.7ex] {\scriptsize$\pm$\,0.09}} & \makecell[tc]{\posgreen{$+$0.01} \\[-0.7ex] {\scriptsize$\pm$\,0.20}} & \makecell[tc]{\negred{$-$0.02} \\[-0.7ex] {\scriptsize$\pm$\,0.66}} & \makecell[tc]{\negred{$-$0.40} \\[-0.7ex] {\scriptsize$\pm$\,2.92}} & \makecell[tc]{\negred{$-$1.61} \\[-0.7ex] {\scriptsize$\pm$\,6.00}} & \makecell[tc]{\negred{$-$0.11} \\[-0.7ex] {\scriptsize$\pm$\,0.54}} & \makecell[tc]{\negred{$-$0.02} \\[-0.7ex] {\scriptsize$\pm$\,0.36}} & \makecell[tc]{\posgreen{$+$0.04} \\[-0.7ex] {\scriptsize$\pm$\,0.29}} \\
        SAGE & \makecell[tc]{\negred{$-$0.42} \\[-0.7ex] {\scriptsize$\pm$\,1.35}} & \makecell[tc]{\textbf{\negred{$-$2.52}} \\[-0.7ex] {\scriptsize$\pm$\,1.93}} & \makecell[tc]{\posgreen{$+$0.54} \\[-0.7ex] {\scriptsize$\pm$\,1.42}} & \makecell[tc]{\negred{$-$0.06} \\[-0.7ex] {\scriptsize$\pm$\,0.12}} & \makecell[tc]{\posgreen{$+$0.13} \\[-0.7ex] {\scriptsize$\pm$\,0.48}} & \makecell[tc]{\negred{$-$0.08} \\[-0.7ex] {\scriptsize$\pm$\,0.17}} & \makecell[tc]{\textbf{\posgreen{$+$0.06}} \\[-0.7ex] {\scriptsize$\pm$\,0.05}} & \makecell[tc]{\posgreen{$+$0.13} \\[-0.7ex] {\scriptsize$\pm$\,0.31}} & \makecell[tc]{\posgreen{$+$0.02} \\[-0.7ex] {\scriptsize$\pm$\,3.00}} & \makecell[tc]{\posgreen{$+$0.63} \\[-0.7ex] {\scriptsize$\pm$\,7.31}} & \makecell[tc]{\posgreen{$+$0.21} \\[-0.7ex] {\scriptsize$\pm$\,0.59}} & \makecell[tc]{\negred{$-$0.11} \\[-0.7ex] {\scriptsize$\pm$\,0.33}} & --- \\
        GAT & \makecell[tc]{\posgreen{$+$0.74} \\[-0.7ex] {\scriptsize$\pm$\,1.27}} & \makecell[tc]{\negred{$-$0.68} \\[-0.7ex] {\scriptsize$\pm$\,0.92}} & \makecell[tc]{\posgreen{$+$0.26} \\[-0.7ex] {\scriptsize$\pm$\,1.24}} & \makecell[tc]{\posgreen{$+$0.06} \\[-0.7ex] {\scriptsize$\pm$\,0.31}} & \makecell[tc]{\textbf{\posgreen{$+$0.07}} \\[-0.7ex] {\scriptsize$\pm$\,0.07}} & \makecell[tc]{\negred{$-$0.08} \\[-0.7ex] {\scriptsize$\pm$\,0.14}} & \makecell[tc]{\posgreen{$+$0.03} \\[-0.7ex] {\scriptsize$\pm$\,0.08}} & \makecell[tc]{\negred{$-$0.01} \\[-0.7ex] {\scriptsize$\pm$\,1.07}} & \makecell[tc]{\posgreen{$+$0.90} \\[-0.7ex] {\scriptsize$\pm$\,2.58}} & \makecell[tc]{\negred{$-$0.81} \\[-0.7ex] {\scriptsize$\pm$\,6.97}} & \makecell[tc]{\negred{$-$0.27} \\[-0.7ex] {\scriptsize$\pm$\,0.45}} & \makecell[tc]{\negred{$-$26.21} \\[-0.7ex] {\scriptsize$\pm$\,44.01}} & \makecell[tc]{\posgreen{$+$0.29} \\[-0.7ex] {\scriptsize$\pm$\,0.78}} \\
        \bottomrule
        \end{tabular}
    }
\end{table}

\paragraph{Meta-learned sharpening coefficient.}
\label{par:meta-learned}
We consider a dynamic variant in which the sharpening coefficient is adapted during training. At each epoch, we randomly split the labeled training nodes into an inner-training subset and a held-out meta-training subset. The model parameters are first updated differentiably on the inner-training subset using the transductive sharpening objective. We then update $\lambda$ so that the one-step-updated model minimizes the supervised loss on the held-out meta-training subset. As shown in Table~\ref{tab:meta_lambda}, this adaptive strategy performs worse than the simpler validation-selected constant $\lambda$, suggesting that the additional meta-optimization introduces instability or noise that is not offset by better coefficient selection.
\begin{table}[H]
\centering
\caption{
Test-accuracy difference between the meta-learned-$\lambda$ variant and the constant-$\lambda$ transductive sharpening baseline. Negative values indicate that adapting $\lambda$ during training performs worse than selecting a fixed value by validation. \posgreen{Green} when positive, \negred{red} when negative; \textbf{bold} when $|\Delta| > \sigma$. Meta-learning the $\lambda$ value shows no advantage to considering a fixed value.
}
\label{tab:meta_lambda}
\fittotextwidth{%
    \setlength{\tabcolsep}{3pt}%
    \begin{tabular}{l ccccccccccccc}
    \toprule
     & Cora & CiteSeer & PubMed & Computer & Photo & CS & Physics & WikiCS & Squirrel & Chameleon & Amazon-Rat. & Roman-Emp. & Minesweeper \\
     \cmidrule(lr){2-9} \cmidrule(lr){10-14}
    GCN & \makecell[tc]{\negred{$-$0.58} \\[-0.7ex] {\scriptsize$\pm$\,0.74}} & \makecell[tc]{\negred{$-$0.42} \\[-0.7ex] {\scriptsize$\pm$\,0.45}} & \makecell[tc]{\posgreen{$+$0.06} \\[-0.7ex] {\scriptsize$\pm$\,0.39}} & \makecell[tc]{\negred{$-$0.19} \\[-0.7ex] {\scriptsize$\pm$\,0.27}} & \makecell[tc]{\textbf{\negred{$-$1.35}} \\[-0.7ex] {\scriptsize$\pm$\,0.76}} & \makecell[tc]{\negred{$-$0.06} \\[-0.7ex] {\scriptsize$\pm$\,0.11}} & \makecell[tc]{0 \\[-0.7ex] {\scriptsize$\pm$\,0.14}} & \makecell[tc]{\negred{$-$0.23} \\[-0.7ex] {\scriptsize$\pm$\,0.72}} & \makecell[tc]{\posgreen{$+$0.16} \\[-0.7ex] {\scriptsize$\pm$\,2.83}} & \makecell[tc]{\negred{$-$0.25} \\[-0.7ex] {\scriptsize$\pm$\,5.82}} & \makecell[tc]{\negred{$-$0.57} \\[-0.7ex] {\scriptsize$\pm$\,0.88}} & \makecell[tc]{\textbf{\negred{$-$0.50}} \\[-0.7ex] {\scriptsize$\pm$\,0.27}} & \makecell[tc]{\negred{$-$0.38} \\[-0.7ex] {\scriptsize$\pm$\,0.48}} \\
    SAGE & \makecell[tc]{\negred{$-$1.28} \\[-0.7ex] {\scriptsize$\pm$\,1.57}} & \makecell[tc]{\textbf{\negred{$-$2.16}} \\[-0.7ex] {\scriptsize$\pm$\,0.47}} & \makecell[tc]{\textbf{\negred{$-$1.78}} \\[-0.7ex] {\scriptsize$\pm$\,1.14}} & \makecell[tc]{\textbf{\negred{$-$0.39}} \\[-0.7ex] {\scriptsize$\pm$\,0.20}} & \makecell[tc]{\negred{$-$0.32} \\[-0.7ex] {\scriptsize$\pm$\,0.42}} & \makecell[tc]{\posgreen{$+$0.07} \\[-0.7ex] {\scriptsize$\pm$\,0.15}} & \makecell[tc]{\negred{$-$0.03} \\[-0.7ex] {\scriptsize$\pm$\,0.05}} & \makecell[tc]{\negred{$-$0.63} \\[-0.7ex] {\scriptsize$\pm$\,0.70}} & \makecell[tc]{\posgreen{$+$0.40} \\[-0.7ex] {\scriptsize$\pm$\,3.11}} & \makecell[tc]{\negred{$-$1.05} \\[-0.7ex] {\scriptsize$\pm$\,8.67}} & \makecell[tc]{\textbf{\negred{$-$1.84}} \\[-0.7ex] {\scriptsize$\pm$\,0.50}} & \makecell[tc]{\textbf{\negred{$-$0.53}} \\[-0.7ex] {\scriptsize$\pm$\,0.39}} & \makecell[tc]{\posgreen{$+$0.49} \\[-0.7ex] {\scriptsize$\pm$\,1.07}} \\
    GAT & \makecell[tc]{\textbf{\negred{$-$1.90}} \\[-0.7ex] {\scriptsize$\pm$\,1.75}} & \makecell[tc]{\negred{$-$0.66} \\[-0.7ex] {\scriptsize$\pm$\,0.70}} & \makecell[tc]{\negred{$-$0.28} \\[-0.7ex] {\scriptsize$\pm$\,1.46}} & \makecell[tc]{0 \\[-0.7ex] {\scriptsize$\pm$\,0.46}} & \makecell[tc]{\negred{$-$0.12} \\[-0.7ex] {\scriptsize$\pm$\,0.23}} & \makecell[tc]{\textbf{\posgreen{$+$0.13}} \\[-0.7ex] {\scriptsize$\pm$\,0.07}} & \makecell[tc]{\textbf{\negred{$-$0.15}} \\[-0.7ex] {\scriptsize$\pm$\,0.05}} & \makecell[tc]{\textbf{\negred{$-$1.13}} \\[-0.7ex] {\scriptsize$\pm$\,0.61}} & \makecell[tc]{\posgreen{$+$1.57} \\[-0.7ex] {\scriptsize$\pm$\,2.19}} & \makecell[tc]{\negred{$-$1.05} \\[-0.7ex] {\scriptsize$\pm$\,5.47}} & \makecell[tc]{\negred{$-$0.34} \\[-0.7ex] {\scriptsize$\pm$\,0.48}} & \makecell[tc]{\textbf{\negred{$-$0.43}} \\[-0.7ex] {\scriptsize$\pm$\,0.41}} & \makecell[tc]{\textbf{\negred{$-$1.00}} \\[-0.7ex] {\scriptsize$\pm$\,0.54}} \\
    \bottomrule
    \end{tabular}
}
\end{table}

\section{Related Work}

\textbf{Graph Neural Networks.}
Graph neural networks (GNNs) have become the standard approach for node classification~\cite{tuned-gnn}, with architectures such as GCN~\cite{kipf2017semisupervisedclassificationgraphconvolutional}, GAT~\cite{gat}, and GraphSAGE~\cite{hamilton2018inductiverepresentationlearninglarge}. Recent work has shown that well-tuned implementations of these models remain highly competitive~\cite{tuned-gnn}, suggesting that performance gains are not solely driven by architectural innovation.

\textbf{Improved training of GNNs.}
A line of work has explored improving GNN performance through creating different training policies and losses. Common approaches include dropout-based methods~\cite{JMLR:v15:srivastava14a}, structural perturbations such as DropEdge~\cite{drop-edge} and DropNode~\cite{Do_2021}, normalization techniques like PairNorm~\cite{zhao2020pairnormtacklingoversmoothinggnns}, and data augmentation strategies including Mixup~\cite{zhang2018mixupempiricalriskminimization, verma2019manifoldmixupbetterrepresentations}, and G-Mixup~\cite{han2022gmixupgraphdataaugmentation}. While these methods improve GNN training, they do not explicitly leverage predictions on unlabeled nodes during training.

\textbf{Leveraging unlabeled data in GNN training.}
Recent methods explicitly incorporate both labeled and unlabeled data to improve GNN training. InfoGraph~\cite{sun2020infographunsupervisedsemisupervisedgraphlevel} learns representations by maximizing mutual information between local and global views of the graph. Related approaches leverage contrastive learning to exploit unlabeled data~\cite{you2021graphcontrastivelearningaugmentations}, while others rely on augmentation-based regularization such as consistency and diversity objectives~\cite{Bo_Hu_Wang_Zhang_Shi_Zhou_2022}. Mixup-based methods incorporate unlabeled data through interpolation strategies~\cite{10.1145/3442381.3449796, verma2021graphmix}, whereas~\cite{every-node-counts} incorporate unlabeled nodes through information-theoretic objectives and additional regularization terms. Our approach differs from these methods in two key aspects. First, rather than introducing auxiliary objectives or augmentations, we operate directly on the model's predictive distribution. Second, our method requires no architectural changes and integrates seamlessly into standard training.

\textbf{Entropy and confidence-based objectives.}
Entropy plays a central role in many areas of machine learning as a measure of uncertainty. In semi-supervised learning, entropy minimization~\cite{ssl-entropy-minimization} has been used to encourage confident predictions on unlabeled data and in supervised learning, entropy maximization~\cite{pereyra2017regularizing} has been used on labeled data. Another approach widely used to prevent over-confidence in training is label smoothing~\cite{szegedy2016rethinking,muller2019does}. However, in graph neural networks, standard training objectives do not explicitly leverage entropy on unlabeled nodes, despite predictions being available for the entire graph. Our work revisits this gap and shows that directly shaping prediction entropy provides a simple and effective improvement.

\section{Conclusions and Future Work}

In this work, we introduced Transductive Sharpening (TS), a simple loss-level modification that exploits the predictions models already produce on unlabeled nodes during transductive graph learning. TS adds only a single scalar hyperparameter, requires no architectural changes, and improves or matches the corresponding baseline across architectures and benchmarks. These results show that TS offers a strong performance-complexity trade-off by delivering meaningful gains while requiring no substantive changes to the model architecture, data pipeline, or training procedure. More broadly, our findings suggest that unlabeled predictions, which are typically discarded by the standard supervised objective, provide a useful and underexploited signal for transductive graph learning, and that confidence-based estimates offer an effective way to extract it.

\paragraph{Future work.}
This work focuses on transductive node classification with standard GNN and MLP backbones, where Transductive Sharpening can be applied as a lightweight loss-level modification. While our experiments show that TS improves or matches the baseline across a broad set of benchmarks, an immediate extension is to study whether the same objective remains effective in closely related transductive settings, such as link prediction or temporal node classification. Another useful direction is to make the sharpening strength mildly adaptive, for example by varying it across training stages or by using simple confidence-based criteria. We hope this work inspires new transductive graph-learning objectives that leverage unlabeled predictions produced during training to obtain gains in downstream tasks without adding architectural complexity.

\section*{Acknowledgements}
The authors thank Petar Veličković for helpful discussions and feedback. Mar Gonzàlez I Català acknowledges that this project was supported by G-Research. Ferran Hernandez Caralt acknowledges that the project that gave rise to these results received the support of a fellowship from “la Caixa” Foundation (ID 100010434). The fellowship code is LCF/BQ/PFA25/11000012.

\section*{Broader Impacts}
\label{sec:broader_impacts}

This paper presents work whose goal is to advance the field of Machine Learning. The method is foundational graph-learning research with no direct high-risk application.

\newpage

{\small\printbibliography}
\FloatBarrier

\appendix
\newpage
\section{Datasets and Experimental Details}
\label{app:experimental_details}

\subsection{Computing Environment}\label{subsec:compute_env}

Our implementation is built upon tunedGNN~\cite{tuned-gnn}, which is based on PyG~\cite{pyg} and DGL~\cite{dgl}. The experiments were conducted on a single workstation with an RTX 5090. With 3 datasets running in parallel, total compute time to run all 39 (model, dataset)-pairs averaged 1\,h 7\,min.

\subsection{Declaration of Hyperparameters}\label{subsec:reproducibility}

\paragraph{TS variants.} All TS variants inherit their per-cell backbone hyperparameters (depth, width, learning rate, dropout, normalisation, residual connections, training epochs) from TunedGNN~\cite{tuned-gnn} Tables~8--10, reproduced here as Tables~\ref{tab:hparams_meta_lambda_gcn}, \ref{tab:hparams_meta_lambda_gat}, and \ref{tab:hparams_meta_lambda_sage}. The optimal lambda we found is reported in \(\lambda^\star\).

\begin{table}[H]
    \centering
    \scriptsize
    \caption{TunedGNN hyperparameters for GCN.}
    \label{tab:hparams_meta_lambda_gcn}
    \begin{tabular}{lcccccccc}
    \toprule
    Dataset & ResNet & Normalization & Dropout rate & GNNs layer $L$ & Hidden dim & LR & epoch & $\lambda^\star$ \\
    \midrule
    Cora & False & False & 0.7 & 3 & 512 & 0.001 & 500 & 1.35 \\
    CiteSeer & False & False & 0.5 & 2 & 512 & 0.001 & 500 & 0.15 \\
    PubMed & False & False & 0.7 & 2 & 256 & 0.005 & 500 & 0.3 \\
    Computer & False & LN & 0.5 & 3 & 512 & 0.001 & 1000 & 0.65 \\
    Photo & True & LN & 0.5 & 6 & 256 & 0.001 & 1000 & 0.35 \\
    CS & True & LN & 0.3 & 2 & 512 & 0.001 & 1500 & 0.6 \\
    Physics & True & LN & 0.3 & 2 & 64 & 0.001 & 1500 & 0.25 \\
    WikiCS & False & LN & 0.5 & 3 & 256 & 0.001 & 1000 & 0.8 \\
    \midrule
    Squirrel & True & BN & 0.7 & 4 & 256 & 0.01 & 500 & 0.45 \\
    Chameleon & False & False & 0.2 & 5 & 512 & 0.005 & 200 & 0.25 \\
    Amazon-Ratings & True & BN & 0.5 & 4 & 512 & 0.001 & 2500 & 0.5 \\
    Roman-Empire & True & BN & 0.5 & 9 & 512 & 0.001 & 2500 & 0.3 \\
    Minesweeper & True & BN & 0.2 & 12 & 64 & 0.01 & 2000 & 0.1 \\
    \bottomrule
    \end{tabular}

\end{table}

\begin{table}[H]
    \centering
    \scriptsize
    \caption{TunedGNN hyperparameters for GAT.}
    \label{tab:hparams_meta_lambda_gat}
    \begin{tabular}{lcccccccc}
    \toprule
    Dataset & ResNet & Normalization & Dropout rate & GNNs layer $L$ & Hidden dim & LR & epoch & $\lambda^\star$ \\
    \midrule
    Cora & True & False & 0.2 & 3 & 512 & 0.001 & 500 & 1.25 \\
    CiteSeer & True & False & 0.5 & 3 & 256 & 0.001 & 500 & 0.1 \\
    PubMed & False & False & 0.5 & 2 & 512 & 0.01 & 500 & 1.55 \\
    Computer & False & LN & 0.5 & 2 & 64 & 0.001 & 1000 & 0.1 \\
    Photo & True & LN & 0.5 & 3 & 64 & 0.001 & 1000 & 0.4 \\
    CS & True & LN & 0.3 & 1 & 256 & 0.001 & 1500 & 0 \\
    Physics & True & BN & 0.7 & 2 & 256 & 0.001 & 1500 & 0.5 \\
    WikiCS & True & LN & 0.7 & 2 & 512 & 0.001 & 1000 & 0.9 \\
    \midrule
    Squirrel & True & BN & 0.5 & 7 & 512 & 0.005 & 500 & 1 \\
    Chameleon & True & BN & 0.7 & 2 & 256 & 0.01 & 200 & 0.3 \\
    Amazon-Ratings & True & BN & 0.5 & 4 & 512 & 0.001 & 2500 & 0.3 \\
    Roman-Empire & True & BN & 0.3 & 10 & 512 & 0.001 & 2500 & 0.15 \\
    Minesweeper & True & BN & 0.2 & 15 & 64 & 0.01 & 2000 & 0 \\
    \bottomrule
    \end{tabular}

\end{table}

\begin{table}[H]
    \centering
    \scriptsize
    \caption{TunedGNN hyperparameters for SAGE.}
    \label{tab:hparams_meta_lambda_sage}
    \begin{tabular}{lcccccccc}
    \toprule
    Dataset & ResNet & Normalization & Dropout rate & GNNs layer $L$ & Hidden dim & LR & epoch & $\lambda^\star$ \\
    \midrule
    Cora & False & False & 0.7 & 3 & 256 & 0.001 & 500 & 0.65 \\
    CiteSeer & False & False & 0.2 & 3 & 512 & 0.001 & 500 & 0.05 \\
    PubMed & False & False & 0.5 & 4 & 512 & 0.005 & 500 & 1.65 \\
    Computer & False & LN & 0.3 & 4 & 64 & 0.001 & 1000 & 0.4 \\
    Photo & True & LN & 0.2 & 6 & 64 & 0.001 & 1000 & 0.45 \\
    CS & True & LN & 0.5 & 2 & 512 & 0.001 & 1500 & 0.05 \\
    Physics & True & BN & 0.7 & 2 & 64 & 0.001 & 1500 & 0.25 \\
    WikiCS & False & LN & 0.7 & 2 & 256 & 0.001 & 1000 & 1.55 \\
    \midrule
    Squirrel & True & BN & 0.7 & 3 & 256 & 0.01 & 500 & 0.75 \\
    Chameleon & True & BN & 0.7 & 4 & 256 & 0.01 & 200 & 0.15 \\
    Amazon-Ratings & True & BN & 0.5 & 9 & 512 & 0.001 & 2500 & 0.9 \\
    Roman-Empire & False & BN & 0.3 & 9 & 256 & 0.001 & 2500 & 0.35 \\
    Minesweeper & True & BN & 0.2 & 15 & 64 & 0.01 & 2000 & 0.4 \\
    \bottomrule
    \end{tabular}

\end{table}

\paragraph{Unlabeled-node baselines.}
For the comparison in Table~\ref{tab:unlabeled_baselines}, all methods use identical backbones, data splits, and five paired seeds. Self-training is validation-tuned independently for each model--dataset pair over eight combinations of confidence threshold and number of rounds; each round adds confidence-thresholded pseudo-labels and retrains the backbone. TS uses the fixed $\lambda=0.25$ without per-cell tuning and adds only its loss term to the standard training run.

\paragraph{TS with retune.} The $\lambda$ sweep with retune variant in Appendix~\ref{subsec:hparam_retuning_ablation} inherits the same backbone hyperparameters and replaces $(\mathrm{lr}, \mathrm{dropout})$ with the $5{\times}5$ retune (Table~\ref{tab:hparams_lambda_sweep_retune_val}), all other settings are unchanged.

\begin{table}[H]
    \centering
    \scriptsize
    \caption{Hyperparameters for $\lambda$ sweep with retune.}
    \label{tab:hparams_lambda_sweep_retune_val}
    \begin{tabular}{r ll ll ll}
    \toprule
    & \multicolumn{2}{c}{GCN} & \multicolumn{2}{c}{SAGE} & \multicolumn{2}{c}{GAT} \\
    \cmidrule(lr){2-3} \cmidrule(lr){4-5} \cmidrule(lr){6-7}
    Dataset & lr & dropout & lr & dropout & lr & dropout \\
    \midrule
    Cora & 0.001 & 0.65 & 0.00025 & 0.7 & 0.0005 & 0.3 \\
    CiteSeer & 0.004 & 0.6 & 0.001 & 0.3 & 0.0005 & 0.55 \\
    PubMed & 0.01 & 0.8 & 0.0025 & 0.45 & 0.01 & 0.5 \\
    Computer & 0.001 & 0.55 & 0.0005 & 0.2 & 0.004 & 0.6 \\
    Photo & 0.001 & 0.4 & 0.0005 & 0.3 & 0.001 & 0.5 \\
    CS & 0.002 & 0.3 & 0.00025 & 0.4 & 0.00025 & 0.35 \\
    Physics & 0.0005 & 0.3 & 0.001 & 0.6 & 0.004 & 0.7 \\
    WikiCS & 0.0005 & 0.5 & 0.004 & 0.75 & 0.00025 & 0.8 \\
    \midrule
    Squirrel & 0.02 & 0.6 & 0.02 & 0.75 & 0.005 & 0.4 \\
    Chameleon & 0.02 & 0.15 & 0.02 & 0.7 & 0.01 & 0.75 \\
    Amazon-Ratings & 0.004 & 0.6 & 0.001 & 0.45 & 0.001 & 0.6 \\
    Roman-Empire & 0.004 & 0.4 & 0.004 & 0.25 & 0.001 & 0.3 \\
    Minesweeper & 0.02 & 0.25 & 0.02 & 0.2 & 0.01 & 0.2 \\
    \bottomrule
    \end{tabular}

\end{table}

\paragraph{MLP baselines.} For the MLP baselines and their TS variants, we use a single fixed architecture and optimization configuration across all datasets (Table~\ref{tab:hparams_mlp_fixed}). This avoids dataset-specific MLP tuning and isolates the effect of adding Transductive Sharpening to a feature-only model.

\paragraph{Meta-learned lambda.} For the meta-learned \(\lambda\) variant in Ablation~\ref{par:meta-learned}, we additionally use the fixed meta-learner hyperparameters listed in Table~\ref{tab:hparams_meta_lambda_constants}.

\begin{table}[htbp]
    \centering
    \scriptsize
    \begin{minipage}[t]{0.48\textwidth}
    \centering
    \caption{Fixed MLP hyperparameters used across all datasets.}
    \label{tab:hparams_mlp_fixed}
    \begin{tabular}{rl}
        \toprule
        Hyperparameter & Value \\
        \midrule
        Hidden channels & 512 \\
        Training epochs & 1000 \\
        Learning rate & 0.001 \\
        Local layers & 3 \\
        Weight decay & 0.0005 \\
        Dropout & 0.5 \\
        Tsallis \(q\) & 2.0 \\
        Seeds (runs) per cell & 5 \\
        \bottomrule
    \end{tabular}
    \end{minipage}\hfill
    \begin{minipage}[t]{0.48\textwidth}
    \centering
    \caption{Meta-learner constants.}
    \label{tab:hparams_meta_lambda_constants}
    \begin{tabular}{rl}
      \toprule
      Meta-learner constant & Value \\
      \midrule
      $\lambda_{\max}$ & 1 \\
      Meta-LR & 0.001 \\
      Meta-hidden dim & 32 \\
      Meta-warmup epochs & 20 \\
      Tsallis $q$ & 2 \\
      Seeds (runs) per cell & 5 \\
      \bottomrule
    \end{tabular}
    \end{minipage}
\end{table}

\subsection{Declaration of Splits}

\begin{itemize}
    \item \textbf{Cora, CiteSeer, PubMed}: a single random class-balanced split (20 train per class, 500 val, 1000 test), drawn once at a fixed seed and reused across all 5 runs; the reported standard deviation on these datasets therefore reflects initialization variance only.
    \item \textbf{Amazon-Computer, Amazon-Photo, Coauthor-CS, Coauthor-Physics}~\cite{shchur2018pitfalls}: a single fixed $60/20/20$ split per dataset.
    \item \textbf{WikiCS}~\cite{mernyei2020wiki}: the dataset provides 20 splits; our 5-run results use the first 5.
    \item \textbf{Squirrel, Chameleon}~\cite{pei2020geom}: the filtered (updated version) versions of Platonov et al.~\cite{platonov2023critical} (${\sim}48/32/20$ after filtering) provide 10 splits; our 5-run results use the first 5.
    \item \textbf{Roman-Empire, Amazon-Ratings, Minesweeper}~\cite{platonov2023critical}: 10 fixed $50/25/25$ splits are provided; our 5-run results use the first 5.
\end{itemize}

Where multiple splits are available, each of the 5 runs uses one split (run $r$ uses split $r$); the $\lambda$-selection sweeps used the first 3 splits.

\FloatBarrier

\section{Additional Results}
\label{app:additional_results}

\subsection{Hyperparameter retuning}\label{subsec:hparam_retuning_ablation}

In a separate diagnostic, we additionally re-tune $(\mathrm{lr}, \mathrm{dropout})$ on a $5\!\times\!5$ grid centered on TunedGNN's per-cell defaults at the val-best $\lambda^\star$. Table~\ref{tab:hparam_changes} reports the resulting test-accuracy shifts relative to the same cell with TunedGNN's defaults: most cells move by less than one combined std, suggesting that the gains reported in the main results would not increase substantially under this additional tuning.

\begin{table}[H]
    \centering
    \caption{Test-accuracy gain attributable to the $(\mathrm{lr}, \mathrm{dropout})$ retune, relative to the same cell pre-retune (i.e., at the $\lambda^\star$ but with TunedGNN's default $\mathrm{lr}$ and $\mathrm{dropout}$).  \posgreen{Green} when positive, \negred{red} when negative; \textbf{bold} when $|\Delta| > \sigma$.}
    \label{tab:hparam_changes}
    \fittotextwidth{%
        \setlength{\tabcolsep}{3pt}%
        \begin{tabular}{l ccccccccccccc}
        \toprule
         & Cora & CiteSeer & PubMed & Computer & Photo & CS & Physics & WikiCS & Squirrel & Chameleon & Amazon-Rat. & Roman-Emp. & Minesweeper \\
         \cmidrule(lr){2-9} \cmidrule(lr){10-14}
        GCN & \makecell[tc]{\negred{$-$0.12} \\[-0.7ex] {\scriptsize$\pm$\,0.73}} & \makecell[tc]{\negred{$-$0.08} \\[-0.7ex] {\scriptsize$\pm$\,0.64}} & \makecell[tc]{\posgreen{$+$0.18} \\[-0.7ex] {\scriptsize$\pm$\,0.91}} & \makecell[tc]{\negred{$-$0.25} \\[-0.7ex] {\scriptsize$\pm$\,0.32}} & \makecell[tc]{\negred{$-$0.07} \\[-0.7ex] {\scriptsize$\pm$\,0.33}} & \makecell[tc]{\negred{$-$0.01} \\[-0.7ex] {\scriptsize$\pm$\,0.17}} & \makecell[tc]{\posgreen{$+$0.04} \\[-0.7ex] {\scriptsize$\pm$\,0.09}} & \makecell[tc]{\negred{$-$0.18} \\[-0.7ex] {\scriptsize$\pm$\,0.77}} & \makecell[tc]{\negred{$-$0.04} \\[-0.7ex] {\scriptsize$\pm$\,2.75}} & \makecell[tc]{\negred{$-$0.07} \\[-0.7ex] {\scriptsize$\pm$\,6.57}} & \makecell[tc]{\posgreen{$+$0.41} \\[-0.7ex] {\scriptsize$\pm$\,0.49}} & \makecell[tc]{\textbf{\posgreen{$+$0.36}} \\[-0.7ex] {\scriptsize$\pm$\,0.34}} & \makecell[tc]{\posgreen{$+$0.21} \\[-0.7ex] {\scriptsize$\pm$\,0.25}} \\
        SAGE & \makecell[tc]{\negred{$-$0.16} \\[-0.7ex] {\scriptsize$\pm$\,1.21}} & \makecell[tc]{\textbf{\negred{$-$0.36}} \\[-0.7ex] {\scriptsize$\pm$\,0.32}} & \makecell[tc]{\posgreen{$+$0.38} \\[-0.7ex] {\scriptsize$\pm$\,1.38}} & \makecell[tc]{\posgreen{$+$0.04} \\[-0.7ex] {\scriptsize$\pm$\,0.24}} & \makecell[tc]{\posgreen{$+$0.33} \\[-0.7ex] {\scriptsize$\pm$\,0.36}} & \makecell[tc]{\posgreen{$+$0.11} \\[-0.7ex] {\scriptsize$\pm$\,0.13}} & \makecell[tc]{\posgreen{$+$0.04} \\[-0.7ex] {\scriptsize$\pm$\,0.06}} & \makecell[tc]{\posgreen{$+$0.15} \\[-0.7ex] {\scriptsize$\pm$\,0.44}} & \makecell[tc]{\negred{$-$1.28} \\[-0.7ex] {\scriptsize$\pm$\,2.51}} & \makecell[tc]{\posgreen{$+$0.67} \\[-0.7ex] {\scriptsize$\pm$\,6.97}} & \makecell[tc]{\posgreen{$+$0.11} \\[-0.7ex] {\scriptsize$\pm$\,0.63}} & \makecell[tc]{\textbf{\posgreen{$+$0.46}} \\[-0.7ex] {\scriptsize$\pm$\,0.38}} & \makecell[tc]{\textbf{\posgreen{$+$0.81}} \\[-0.7ex] {\scriptsize$\pm$\,0.52}} \\
        GAT & \makecell[tc]{\negred{$-$1.66} \\[-0.7ex] {\scriptsize$\pm$\,1.70}} & \makecell[tc]{\posgreen{$+$0.38} \\[-0.7ex] {\scriptsize$\pm$\,0.63}} & 0 & \makecell[tc]{\posgreen{$+$0.05} \\[-0.7ex] {\scriptsize$\pm$\,0.32}} & 0 & \makecell[tc]{\posgreen{$+$0.01} \\[-0.7ex] {\scriptsize$\pm$\,0.05}} & \makecell[tc]{\posgreen{$+$0.03} \\[-0.7ex] {\scriptsize$\pm$\,0.06}} & \makecell[tc]{\posgreen{$+$0.11} \\[-0.7ex] {\scriptsize$\pm$\,0.31}} & \makecell[tc]{\posgreen{$+$0.61} \\[-0.7ex] {\scriptsize$\pm$\,3.09}} & \makecell[tc]{\negred{$-$1.89} \\[-0.7ex] {\scriptsize$\pm$\,6.34}} & \makecell[tc]{\posgreen{$+$0.13} \\[-0.7ex] {\scriptsize$\pm$\,0.42}} & 0 & 0 \\
        \bottomrule
        \end{tabular}
    }
\end{table}

\subsection{Results on large-scale graphs}

We further evaluate TS on three large-scale OGB benchmarks. For
ogbn-products and ogbn-papers100M, both the baseline and TS models are trained with mini-batch neighbor sampling. On ogbn-papers100M, we use sampled inference with the same fanout for both the baseline and TS. This lowers the reported absolute accuracy relative to full-neighborhood inference, but does not affect the fairness of the comparison. Table~\ref{tab:other_datasets} shows that TS improves performance on all three OGB benchmarks evaluated, with the largest gains on ogbn-products and ogbn-papers100M.

\begin{table}[htbp]
    \centering
    \caption{Results on large-scale graphs. Means $\pm$ std over 5 paired seeds. ogbn-products and ogbn-papers100M use mini-batch neighbor sampling, and ogbn-papers100M uses sampled inference with fanout $[10,10,10]$.  \posgreen{Green} when positive, \negred{red} when negative; \textbf{bold} when $|\Delta| > \sigma$.}
    \label{tab:other_datasets}
    \fittotextwidth{%
        \setlength{\tabcolsep}{3pt}%
        \begin{tabular}{r ccc !{\hskip 1em} ccc}
            \toprule
             & \multicolumn{3}{c}{Baseline} & \multicolumn{3}{c}{$+$TS ($\lambda{=}0.25$)} \\
            \cmidrule(lr){2-4} \cmidrule(lr){5-7}
            Dataset & GCN & SAGE & GAT & GCN & SAGE & GAT \\
            \midrule
            ogbn-arxiv & \makecell[tc]{$73.09$ \\[-0.7ex] {\scriptsize$\pm$\,0.18}} & \makecell[tc]{$72.48$ \\[-0.7ex] {\scriptsize$\pm$\,0.26}} & \makecell[tc]{$72.38$ \\[-0.7ex] {\scriptsize$\pm$\,0.12}} & \makecell[tc]{$73.30$\,{\scriptsize$\pm$\,0.23} \\ $+$0.21} & \makecell[tc]{$72.72$\,{\scriptsize$\pm$\,0.17} \\ $+$0.24} & \makecell[tc]{$72.53$\,{\scriptsize$\pm$\,0.18} \\ $+$0.15} \\
            ogbn-products & \makecell[tc]{$80.95$ \\[-0.7ex] {\scriptsize$\pm$\,0.34}} & \makecell[tc]{$83.02$ \\[-0.7ex] {\scriptsize$\pm$\,0.39}} & --- & \makecell[tc]{$81.82$\,{\scriptsize$\pm$\,0.30} \\ \textbf{\posgreen{$+$0.87}}} & \makecell[tc]{$84.51$\,{\scriptsize$\pm$\,0.14} \\ \textbf{\posgreen{$+$1.49}}} & --- \\
            ogbn-papers100M & --- & \makecell[tc]{$61.48$ \\[-0.7ex] {\scriptsize$\pm$\,0.26}} & --- & --- & \makecell[tc]{$62.23$\,{\scriptsize$\pm$\,0.15} \\ \textbf{\posgreen{$+$0.75}}} & --- \\
            \bottomrule
        \end{tabular}
    }
    \par\smallskip
\end{table}

\subsection{Transductive sharpening on stronger architectures}

To investigate whether the benefits of TS extend beyond the classical GNN backbones considered in our main experiments, we also apply TS to SGFormer~\cite{sgformer} and Polynormer~\cite{polynormer}. We use the authors' tuned configurations and five paired seeds, and each TS run uses the same configuration as its paired baseline, with a fixed $\lambda=0.25$.

Table~\ref{tab:ts_recent_architectures} reports the paired differences in test performance. TS yields positive mean improvements in seven of the ten architecture--dataset pairs, although the magnitude of these gains varies. These results suggest that TS can also benefit stronger node-classification architectures, but the magnitude of the effect may depend on the particular model and dataset.

\begin{table}[htbp]
    \centering
    \caption{Paired test-performance differences (TS minus baseline), in percentage points. Values are means $\pm$ standard deviations over five paired seeds, using the authors' tuned configurations and a fixed $\lambda=0.25$.}
    \label{tab:ts_recent_architectures}
    \fittotextwidth{%
        \begin{tabular}{lccccc}
            \toprule
             & Cora & CiteSeer & PubMed & Chameleon & Squirrel \\
            \midrule
            SGFormer$+$TS & $+2.18 \pm 0.78$ & $+2.66 \pm 1.38$ & $-0.78 \pm 1.70$ & $-0.47 \pm 0.96$ & $+0.22 \pm 0.27$ \\
            \midrule
             & Roman-Empire & Amazon-Ratings & WikiCS & Minesweeper & Photo \\
            \midrule
            Polynormer$+$TS & $+0.44 \pm 0.30$ & $+0.89 \pm 0.59$ & $-0.08 \pm 0.37$ & $+0.19 \pm 0.42$ & $+0.16 \pm 0.18$ \\
            \bottomrule
        \end{tabular}
    }
\end{table}

\section{Further Ablations}
\label{sec:further-ablations}

\paragraph{Sharpen only on test nodes.}
We evaluate a variant with $\lambda = 0.25$ in which TS is applied only to test nodes, excluding validation nodes from the unlabeled-node sharpening term.
\begin{table}[H]
    \centering
    \caption{Test-accuracy gain of the restricted variant over the standard TS.  \posgreen{Green} when positive, \negred{red} when negative; \textbf{bold} when $|\Delta| > \sigma$.}
    \label{tab:ts_test_nodes_changes}
    \fittotextwidth{%
        \setlength{\tabcolsep}{3pt}%
        \begin{tabular}{l ccccccccccccc}
        \toprule
         & Cora & CiteSeer & PubMed & Computer & Photo & CS & Physics & WikiCS & Squirrel & Chameleon & Amazon-Rat. & Roman-Emp. & Minesweeper \\
         \cmidrule(lr){2-9} \cmidrule(lr){10-14}
        GCN & \makecell[tc]{\negred{$-$0.10} \\[-0.7ex] {\scriptsize$\pm$\,0.51}} & \makecell[tc]{\negred{$-$0.06} \\[-0.7ex] {\scriptsize$\pm$\,0.20}} & \makecell[tc]{\negred{$-$0.16} \\[-0.7ex] {\scriptsize$\pm$\,0.38}} & \makecell[tc]{\posgreen{$+$0.25} \\[-0.7ex] {\scriptsize$\pm$\,0.29}} & \makecell[tc]{\negred{$-$0.02} \\[-0.7ex] {\scriptsize$\pm$\,0.20}} & \makecell[tc]{\textbf{\posgreen{$+$0.10}} \\[-0.7ex] {\scriptsize$\pm$\,0.07}} & \makecell[tc]{\posgreen{$+$0.02} \\[-0.7ex] {\scriptsize$\pm$\,0.15}} & \makecell[tc]{\negred{$-$0.15} \\[-0.7ex] {\scriptsize$\pm$\,0.55}} & \makecell[tc]{\posgreen{$+$0.04} \\[-0.7ex] {\scriptsize$\pm$\,2.70}} & \makecell[tc]{\negred{$-$0.58} \\[-0.7ex] {\scriptsize$\pm$\,6.47}} & \makecell[tc]{\negred{$-$0.04} \\[-0.7ex] {\scriptsize$\pm$\,0.47}} & \makecell[tc]{\textbf{\negred{$-$0.30}} \\[-0.7ex] {\scriptsize$\pm$\,0.28}} & \makecell[tc]{\negred{$-$0.11} \\[-0.7ex] {\scriptsize$\pm$\,0.37}} \\
        SAGE & \makecell[tc]{\negred{$-$0.18} \\[-0.7ex] {\scriptsize$\pm$\,1.27}} & \makecell[tc]{\textbf{\posgreen{$+$1.66}} \\[-0.7ex] {\scriptsize$\pm$\,1.43}} & \makecell[tc]{\posgreen{$+$1.12} \\[-0.7ex] {\scriptsize$\pm$\,1.35}} & \makecell[tc]{\textbf{\negred{$-$0.19}} \\[-0.7ex] {\scriptsize$\pm$\,0.19}} & \makecell[tc]{\posgreen{$+$0.26} \\[-0.7ex] {\scriptsize$\pm$\,0.48}} & \makecell[tc]{\posgreen{$+$0.09} \\[-0.7ex] {\scriptsize$\pm$\,0.11}} & \makecell[tc]{\negred{$-$0.06} \\[-0.7ex] {\scriptsize$\pm$\,0.12}} & \makecell[tc]{\negred{$-$0.25} \\[-0.7ex] {\scriptsize$\pm$\,0.43}} & \makecell[tc]{\posgreen{$+$0.30} \\[-0.7ex] {\scriptsize$\pm$\,3.26}} & \makecell[tc]{\posgreen{$+$1.84} \\[-0.7ex] {\scriptsize$\pm$\,7.61}} & \makecell[tc]{\negred{$-$0.07} \\[-0.7ex] {\scriptsize$\pm$\,0.45}} & \makecell[tc]{\negred{$-$0.22} \\[-0.7ex] {\scriptsize$\pm$\,0.32}} & \makecell[tc]{\posgreen{$+$0.06} \\[-0.7ex] {\scriptsize$\pm$\,0.90}} \\
        GAT & \makecell[tc]{\negred{$-$0.62} \\[-0.7ex] {\scriptsize$\pm$\,1.30}} & \makecell[tc]{\negred{$-$0.04} \\[-0.7ex] {\scriptsize$\pm$\,1.03}} & \makecell[tc]{\negred{$-$0.36} \\[-0.7ex] {\scriptsize$\pm$\,1.93}} & \makecell[tc]{\posgreen{$+$0.11} \\[-0.7ex] {\scriptsize$\pm$\,0.21}} & \makecell[tc]{\negred{$-$0.05} \\[-0.7ex] {\scriptsize$\pm$\,0.12}} & \makecell[tc]{\posgreen{$+$0.03} \\[-0.7ex] {\scriptsize$\pm$\,0.11}} & \makecell[tc]{\negred{$-$0.01} \\[-0.7ex] {\scriptsize$\pm$\,0.08}} & \makecell[tc]{\negred{$-$0.19} \\[-0.7ex] {\scriptsize$\pm$\,1.27}} & \makecell[tc]{\posgreen{$+$1.30} \\[-0.7ex] {\scriptsize$\pm$\,3.46}} & \makecell[tc]{\negred{$-$1.58} \\[-0.7ex] {\scriptsize$\pm$\,5.89}} & \makecell[tc]{0 \\[-0.7ex] {\scriptsize$\pm$\,0.38}} & \makecell[tc]{\textbf{\negred{$-$0.52}} \\[-0.7ex] {\scriptsize$\pm$\,0.51}} & \makecell[tc]{\posgreen{$+$0.40} \\[-0.7ex] {\scriptsize$\pm$\,0.76}} \\
        \bottomrule
        \end{tabular}
    }
\end{table}

Planetoid has nodes that are not in test, val, or train. We continue to sharpen these nodes for the purpose of this ablation.

\section{Further Visualizations}
\label{app:further-visualizations}

This appendix provides additional visualizations of the effect of the sharpening coefficient $\lambda$ by further detailing the trends studied in Section~\ref{sec:lambda-effect}.

Figure~\ref{fig:improvement_regression_per_lambda} reports improvements and regressions in units of the baseline standard deviation, using Glass's $\Delta$ relative to the $\lambda=0$ supervised baseline. The results show that small positive values of $\lambda$ provide the most favorable trade-off: improvements are frequent and often larger than the corresponding regressions. As $\lambda$ increases, regressions become more common and more severe, indicating that aggressive sharpening is less robust. Figure~\ref{fig:lambda_curves} complements this view by showing the full test-accuracy curves for each dataset and backbone. The same qualitative pattern appears across many settings: negative values of $\lambda$ are usually harmful, moderate positive values often improve performance, and overly large values eventually degrade accuracy.

\begin{figure}[H]
    \centering
    \includegraphics[width=0.85\textwidth]{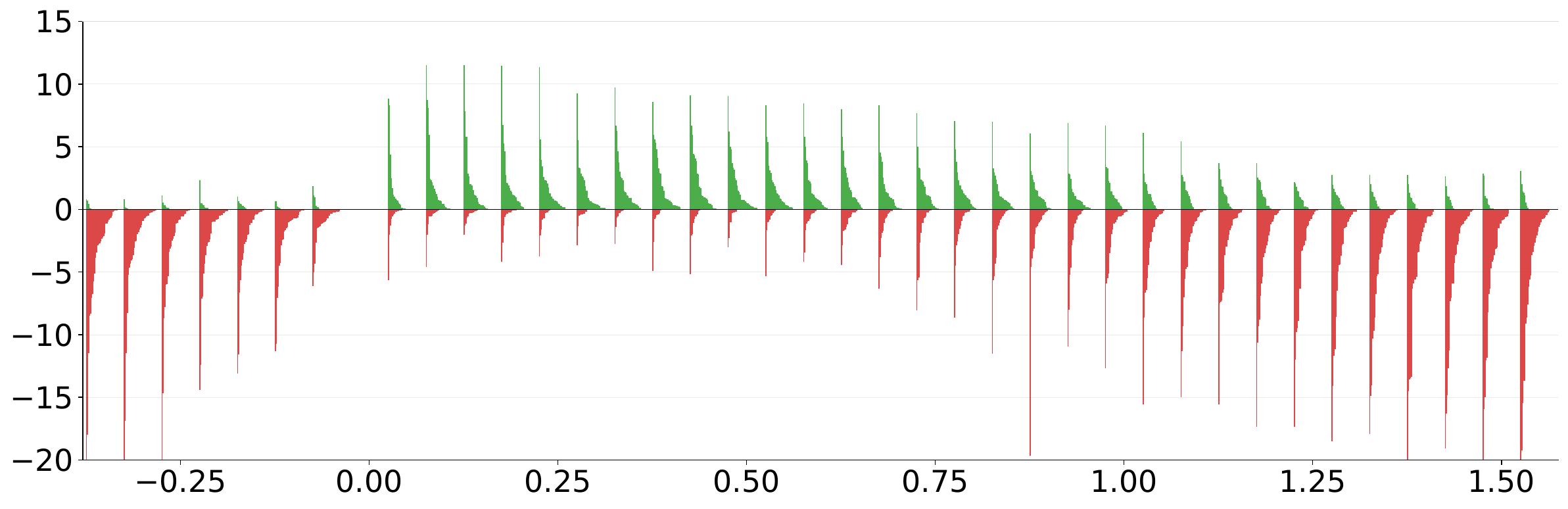}
    \caption{Distribution of improvements and regressions as a function of $\lambda$, measured using Glass's $\Delta$ relative to the $\lambda{=}0$ supervised baseline. Small positive values of $\lambda$ yield the most favorable balance, with improvements occurring frequently and regressions remaining comparatively limited. Larger values of $\lambda$ increase both the frequency and severity of regressions, indicating that aggressive sharpening is less robust across datasets and backbones.}
    \label{fig:improvement_regression_per_lambda}
\end{figure}

\begin{figure}[H]
    \centering
    \includegraphics[width=\textwidth]{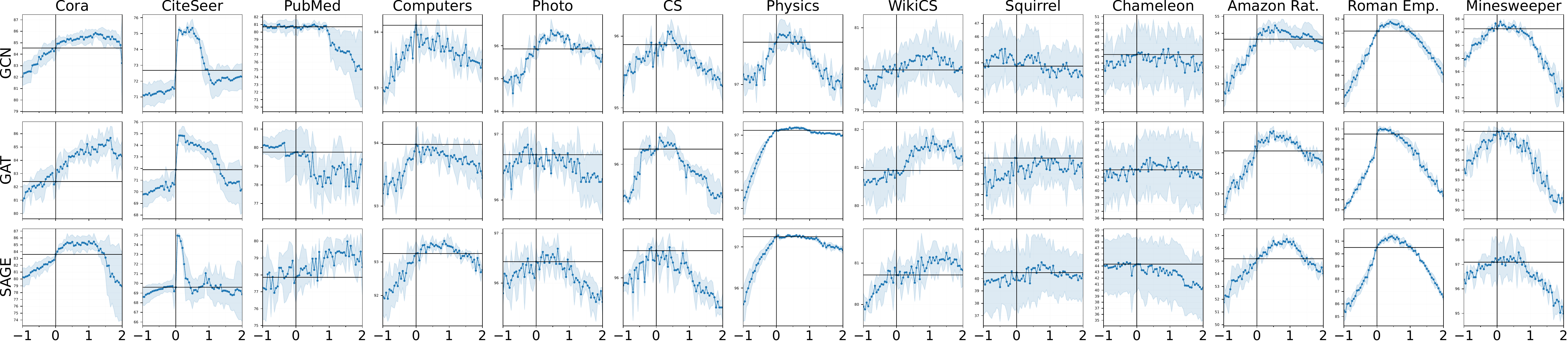}
    \caption{Test performance as a function of $\lambda$ for each dataset and backbone. The crosshair marks the $\lambda{=}0$ supervised baseline. Across many dataset--backbone pairs, performance improves over a finite interval of positive $\lambda$ values before degrading when $\lambda$ becomes too large, while negative values of $\lambda$ are often harmful. This supports the use of moderate positive sharpening and helps explain why a conservative universal value such as $\lambda=0.25$ performs reliably.}
    \label{fig:lambda_curves}
\end{figure}

\section{Proof of Lemma~\ref{lemma:cross-entropy}}
\label{app:proof_cross_entropy}

\begin{proof}
By definition, the cross-entropy between a target distribution $y$ and a prediction $p$ is
\[
\mathcal{L}_{\mathrm{CE}}(y,p)
=
-\sum_{i=1}^C y_i \log p_i .
\]
Adding and subtracting $\sum_{i=1}^C p_i \log p_i$ gives
\[
\mathcal{L}_{\mathrm{CE}}(y,p)
=
-\sum_{i=1}^C p_i \log p_i
+
\sum_{i=1}^C p_i \log p_i
-
\sum_{i=1}^C y_i \log p_i .
\]
The first term is the Shannon entropy $H(p)$. Combining the remaining terms yields
\[
\mathcal{L}_{\mathrm{CE}}(y,p)
=
H(p)
+
\sum_{i=1}^C (p_i-y_i)\log p_i ,
\]
which proves the claim.
\end{proof}

\section{Comparison with competitive architectures}
\label{app:comparison}

This appendix provides a comparison between TS and recent node-classification methods. Our aim is not to claim that a loss-level modification replaces architectural advances, but to contextualize the performance of TS relative to more complex approaches. In Table~\ref{tab:homo_results} and Table~\ref{tab:hetero_results}, we compare TS against competitive architectures across homophilous and heterophilous datasets, respectively.

Overall, TS remains competitive despite using substantially simpler backbone architectures. This is especially clear on the homophilous benchmarks, where standard GNNs augmented with TS often match or approach the performance of more specialized methods. On heterophilous datasets, the comparison is more mixed, as expected, since many competing methods are specifically designed to handle heterophily. Nevertheless, the results show that a simple objective-level modification can recover a meaningful amount of performance without introducing additional architectural machinery.

\begin{table}[p]
    \centering
    \caption{Node classification results over homophilous graphs (\%). Cells are left blank where the original paper used a different evaluation protocol than ours. The top \textbf{\textcolor{customcyan}{$\mathbf{1^{st}}$}}, \textbf{\textcolor{tealblue!90}{$\mathbf{2^{nd}}$}} and \textbf{\textcolor{darkorange!90}{$\mathbf{3^{rd}}$}} results are highlighted.}
    \setlength\tabcolsep{3pt}
    \resizebox{\linewidth}{!}{
    \begin{tabular}{l|llllllll}
        \toprule
            & Cora & CiteSeer & PubMed & Computer & Photo & CS & Physics & WikiCS \\
         \midrule
        \rowcolor{gray!20}
        GraphGPS~\cite{graphgps} & 83.87 {\tiny{± 0.96}} & 72.73 {\tiny{± 1.23}} & 79.94 {\tiny{± 0.26}} & 91.79 {\tiny{± 0.63}} & 94.89 {\tiny{± 0.14}} & 94.04 {\tiny{± 0.21}} & 96.71 {\tiny{± 0.15}} & 78.66 {\tiny{± 0.49}} \\
        NAGphormer~\cite{nagphormer} & 80.92 {\tiny{± 1.17}} & 70.59 {\tiny{± 0.89}} & 80.14 {\tiny{± 1.06}} & 91.69 {\tiny{± 0.30}} & 96.14 {\tiny{± 0.16}} & 95.85 {\tiny{± 0.16}} & 97.35 {\tiny{± 0.12}} & 77.92 {\tiny{± 0.93}} \\
        \rowcolor{gray!20}
        Exphormer~\cite{exphormer} & 83.29 {\tiny{± 1.36}} & 71.85 {\tiny{± 1.11}} & 79.67 {\tiny{± 0.73}} & 91.80 {\tiny{± 0.35}} & 95.69 {\tiny{± 0.39}} & 95.92 {\tiny{± 0.25}} & 97.06 {\tiny{± 0.13}} & 79.38 {\tiny{± 0.62}} \\
        GOAT~\cite{goat} & 83.26 {\tiny{± 1.24}} & 72.21 {\tiny{± 1.29}} & 80.06 {\tiny{± 0.67}} & 92.29 {\tiny{± 0.37}} & 94.33 {\tiny{± 0.21}} & 93.81 {\tiny{± 0.19}} & 96.47 {\tiny{± 0.16}} & 77.96 {\tiny{± 0.63}} \\
        \rowcolor{gray!20}
        NodeFormer~\cite{nodeformer} & 82.73 {\tiny{± 0.75}} & 72.37 {\tiny{± 1.20}} & 79.59 {\tiny{± 0.92}} & 87.29 {\tiny{± 0.58}} & 93.43 {\tiny{± 0.56}} & 95.69 {\tiny{± 0.27}} & 96.48 {\tiny{± 0.34}} & 75.13 {\tiny{± 0.93}} \\
        SGFormer~\cite{sgformer} & 84.82 {\tiny{± 0.85}} & 72.72 {\tiny{± 1.15}} & 80.60 {\tiny{± 0.49}} & 92.42 {\tiny{± 0.66}} & 95.58 {\tiny{± 0.36}} & 95.71 {\tiny{± 0.24}} & 96.75 {\tiny{± 0.26}} & 80.05 {\tiny{± 0.46}} \\
        \rowcolor{gray!20}
        Polynormer~\cite{polynormer} & 83.43 {\tiny{± 0.89}} & 72.19 {\tiny{± 0.83}} & 79.35 {\tiny{± 0.73}} & 93.78 {\tiny{± 0.10}} & 96.57 {\tiny{± 0.23}} & 95.42 {\tiny{± 0.19}} & 97.18 {\tiny{± 0.11}} & 80.26 {\tiny{± 0.92}} \\
        MLP & 60.96 {\tiny{± 2.51}} & 56.58 {\tiny{± 1.14}} & 68.96 {\tiny{± 1.16}} & 82.46 {\tiny{± 0.45}} & 87.57 {\tiny{± 0.52}} & 91.54 {\tiny{± 0.20}} & 95.97 {\tiny{± 0.07}} & 70.96 {\tiny{± 1.00}} \\
        \rowcolor{gray!20}
        NormProp~\cite{normprop} & \textbf{\textcolor{darkorange!90}{85.46 {\tiny{± 0.51}}}} & 74.33 {\tiny{± 0.57}} & \textbf{\textcolor{darkorange!90}{80.72 {\tiny{± 1.09}}}} &  &  &  &  &  \\
        OGFormer~\cite{ogformer} & \textbf{\textcolor{customcyan}{86.40 {\tiny{± 0.30}}}} & 74.70 {\tiny{± 0.50}} & \textbf{\textcolor{customcyan}{81.50 {\tiny{± 0.50}}}} & 92.90 {\tiny{± 0.30}} & 95.50 {\tiny{± 0.00}} & 95.20 {\tiny{± 0.10}} &  &  \\
        \rowcolor{gray!20}
        ELU-GCN~\cite{elu-gcn} & 84.29 {\tiny{± 0.39}} & 74.23 {\tiny{± 0.62}} & 80.51 {\tiny{± 0.21}} &  &  &  &  &  \\
        GraphTARIF~\cite{graph-tarif} &  &  &  & \textbf{\textcolor{customcyan}{94.61 {\tiny{± 0.17}}}} & \textbf{\textcolor{customcyan}{97.03 {\tiny{± 0.19}}}} & \textbf{\textcolor{customcyan}{96.51 {\tiny{± 0.11}}}} & \textbf{\textcolor{tealblue!90}{97.39 {\tiny{± 0.07}}}} & \textbf{\textcolor{darkorange!90}{80.93 {\tiny{± 0.57}}}} \\
        \rowcolor{gray!20}
        GCN~\cite{kipf2017semisupervisedclassificationgraphconvolutional} & 84.54 {\tiny{± 0.86}} & 72.68 {\tiny{± 0.43}} & 80.70 {\tiny{± 0.96}} & \textbf{\textcolor{tealblue!90}{94.12 {\tiny{± 0.08}}}} & 95.90 {\tiny{± 0.33}} & 95.88 {\tiny{± 0.03}} & \textbf{\textcolor{darkorange!90}{97.38 {\tiny{± 0.06}}}} & 79.97 {\tiny{± 0.43}} \\
        SAGE~\cite{hamilton2018inductiverepresentationlearninglarge} & 83.60 {\tiny{± 0.58}} & 69.60 {\tiny{± 0.61}} & 77.86 {\tiny{± 1.49}} & 93.25 {\tiny{± 0.36}} & 96.43 {\tiny{± 0.27}} & \textbf{\textcolor{tealblue!90}{96.29 {\tiny{± 0.12}}}} & 97.25 {\tiny{± 0.08}} & 80.71 {\tiny{± 0.19}} \\
        \rowcolor{gray!20}
        GAT~\cite{gat} & 82.40 {\tiny{± 1.01}} & 71.90 {\tiny{± 0.25}} & 79.76 {\tiny{± 1.21}} & \textbf{\textcolor{darkorange!90}{93.98 {\tiny{± 0.22}}}} & \textbf{\textcolor{tealblue!90}{96.69 {\tiny{± 0.14}}}} & 96.17 {\tiny{± 0.02}} & 97.26 {\tiny{± 0.03}} & 80.92 {\tiny{± 0.58}} \\
        \midrule
        \textbf{MLP+TS} & 64.48 {\tiny{± 2.83}} & 62.72 {\tiny{± 3.20}} & 72.30 {\tiny{± 1.89}} & 82.87 {\tiny{± 0.56}} & 87.65 {\tiny{± 0.22}} & 91.77 {\tiny{± 0.38}} & 95.98 {\tiny{± 0.08}} & 72.48 {\tiny{± 0.83}} \\
        \rowcolor{gray!20}
        \textbf{GCN+TS} & \textbf{\textcolor{tealblue!90}{85.74 {\tiny{± 0.54}}}} & \textbf{\textcolor{customcyan}{75.18 {\tiny{± 0.15}}}} & \textbf{\textcolor{tealblue!90}{80.74 {\tiny{± 0.30}}}} & \textbf{\textcolor{darkorange!90}{93.98 {\tiny{± 0.26}}}} & 96.21 {\tiny{± 0.11}} & 95.89 {\tiny{± 0.06}} & \textbf{\textcolor{customcyan}{97.44 {\tiny{± 0.14}}}} & 80.31 {\tiny{± 0.44}} \\
        \textbf{SAGE+TS} & 85.28 {\tiny{± 1.11}} & \textbf{\textcolor{tealblue!90}{74.96 {\tiny{± 0.24}}}} & 79.72 {\tiny{± 0.69}} & 93.43 {\tiny{± 0.11}} & 96.51 {\tiny{± 0.23}} & \textbf{\textcolor{darkorange!90}{96.24 {\tiny{± 0.10}}}} & 97.23 {\tiny{± <0.01}} & \textbf{\textcolor{tealblue!90}{81.10 {\tiny{± 0.32}}}} \\
        \rowcolor{gray!20}
        \textbf{GAT+TS} & 84.62 {\tiny{± 0.89}} & \textbf{\textcolor{darkorange!90}{74.84 {\tiny{± 0.48}}}} & 78.84 {\tiny{± 0.67}} & 93.86 {\tiny{± 0.13}} & \textbf{\textcolor{darkorange!90}{96.60 {\tiny{± 0.07}}}} & 96.17 {\tiny{± 0.02}} & \textbf{\textcolor{darkorange!90}{97.38 {\tiny{± 0.04}}}} & \textbf{\textcolor{customcyan}{81.78 {\tiny{± 0.22}}}} \\
        \bottomrule
    \end{tabular}
    }
    \label{tab:homo_results}
\end{table}

\begin{table}[p]
    \centering
    \caption{Node classification results over heterophilous graphs (\%). Cells are left blank where the original paper used a different evaluation protocol than ours. The top \textbf{\textcolor{customcyan}{$\mathbf{1^{st}}$}}, \textbf{\textcolor{tealblue!90}{$\mathbf{2^{nd}}$}} and \textbf{\textcolor{darkorange!90}{$\mathbf{3^{rd}}$}} results are highlighted.}
    \setlength\tabcolsep{4pt}
    \resizebox{\linewidth}{!}{
    \begin{tabular}{l|lllll}
        \toprule
            & Squirrel & Chameleon & Amazon-Ratings & Roman-Empire & Minesweeper \\
        Metric & Accuracy$\uparrow$ & Accuracy$\uparrow$ & Accuracy$\uparrow$ & Accuracy$\uparrow$ & ROC-AUC$\uparrow$ \\
        \midrule
        GraphGPS~\cite{graphgps} & 39.81 {\tiny{± 2.28}} & 41.55 {\tiny{± 3.91}} & 53.27 {\tiny{± 0.66}} & 82.72 {\tiny{± 0.68}} & 90.75 {\tiny{± 0.89}} \\
        \rowcolor{gray!20}
        NodeFormer~\cite{nodeformer} & 38.89 {\tiny{± 2.67}} & 36.38 {\tiny{± 3.85}} & 43.79 {\tiny{± 0.57}} & 74.83 {\tiny{± 0.81}} & 87.71 {\tiny{± 0.69}}  \\
        SGFormer~\cite{sgformer} & 42.65 {\tiny{± 2.41}} & 45.21 {\tiny{± 3.72}} & 54.14 {\tiny{± 0.62}} & 80.01 {\tiny{± 0.44}} & 91.42 {\tiny{± 0.41}} \\
        \rowcolor{gray!20}
        Polynormer~\cite{polynormer} & 41.97 {\tiny{± 2.14}} & 41.97 {\tiny{± 3.18}} & 54.96 {\tiny{± 0.22}} & \textbf{\textcolor{darkorange!90}{92.66 {\tiny{± 0.60}}}} & 97.49 {\tiny{± 0.48}} \\
        MLP & 39.30 {\tiny{± 0.79}} & 43.86 {\tiny{± 5.23}} & 48.85 {\tiny{± 0.55}} & 66.10 {\tiny{± 0.44}} & 51.06 {\tiny{± 1.76}} \\
        \rowcolor{gray!20}
        GCN+ReP~\cite{gcn-rep} & \textbf{\textcolor{customcyan}{45.89 {\tiny{± 1.45}}}} & \textbf{\textcolor{customcyan}{47.57 {\tiny{± 3.90}}}} & 52.75 {\tiny{± 0.62}} & 86.43 {\tiny{± 0.74}} & 96.05 {\tiny{± 0.19}} \\
        CoED~\cite{coed} & \textbf{\textcolor{darkorange!90}{45.50 {\tiny{± 1.62}}}} & \textbf{\textcolor{tealblue!90}{47.27 {\tiny{± 3.62}}}} &  & 92.17 {\tiny{± 0.29}} &  \\
        \rowcolor{gray!20}
        M3Dphormer~\cite{m3dphormer} & 44.34 {\tiny{± 1.94}} & \textbf{\textcolor{darkorange!90}{47.09 {\tiny{± 4.05}}}} &  &  & \textbf{\textcolor{darkorange!90}{98.27 {\tiny{± 0.20}}}}  \\
        GraphTARIF~\cite{graph-tarif} & \textbf{\textcolor{tealblue!90}{45.58 {\tiny{± 1.91}}}} &  & \textbf{\textcolor{tealblue!90}{55.86 {\tiny{± 0.42}}}} & \textbf{\textcolor{tealblue!90}{93.23 {\tiny{± 0.38}}}} & \textbf{\textcolor{customcyan}{99.03 {\tiny{± 0.19}}}} \\
        \rowcolor{gray!20}
        BuNN~\cite{bunn} &  &  & 53.74 {\tiny{± 0.51}} & 91.75 {\tiny{± 0.39}} & \textbf{\textcolor{tealblue!90}{98.99 {\tiny{± 0.16}}}} \\
        Dir-Poly~\cite{dir-poly} &  &  & 50.73 {\tiny{± 0.56}} & \textbf{\textcolor{customcyan}{94.51 {\tiny{± 0.22}}}} & 93.74 {\tiny{± 0.70}} \\
        \rowcolor{gray!20}
        CPGNN~\cite{cpgnn} & 30.04 {\tiny{± 2.03}} & 33.00 {\tiny{± 3.15}} & 39.79 {\tiny{± 0.77}} & 63.96 {\tiny{± 0.62}} & 52.03 {\tiny{± 5.46}} \\
        FSGNN~\cite{fsgnn} & 35.92 {\tiny{± 1.32}} & 40.61 {\tiny{± 2.97}} & 52.74 {\tiny{± 0.83}} & 79.92 {\tiny{± 0.56}} & 90.08 {\tiny{± 0.70}} \\
        \rowcolor{gray!20}
        GloGNN~\cite{glognn} & 35.11 {\tiny{± 1.24}} & 25.90 {\tiny{± 3.58}} & 36.89 {\tiny{± 0.14}} & 59.63 {\tiny{± 0.69}} & 51.08 {\tiny{± 1.23}} \\
        GPRGNN~\cite{gprgnn} & 38.95 {\tiny{± 1.99}} & 39.93 {\tiny{± 3.30}} & 44.88 {\tiny{± 0.34}} & 64.85 {\tiny{± 0.27}} & 86.24 {\tiny{± 0.61}} \\
        \rowcolor{gray!20}
        H2GCN~\cite{h2gcn} & 35.10 {\tiny{± 1.15}} & 26.75 {\tiny{± 3.64}} & 36.47 {\tiny{± 0.23}} & 60.11 {\tiny{± 0.52}} & 89.71 {\tiny{± 0.31}} \\
        GCN~\cite{kipf2017semisupervisedclassificationgraphconvolutional} & 43.75 {\tiny{± 1.91}} & 45.30 {\tiny{± 2.30}} & 53.64 {\tiny{± 0.54}} & 91.15 {\tiny{± 0.20}} & 97.26 {\tiny{± 0.22}} \\
        \rowcolor{gray!20}
        GAT~\cite{gat} & 41.51 {\tiny{± 2.34}} & 43.07 {\tiny{± 5.25}} & 55.09 {\tiny{± 0.19}} & 90.49 {\tiny{± 0.22}} & 97.86 {\tiny{± 0.37}} \\
        SAGE~\cite{hamilton2018inductiverepresentationlearninglarge} & 40.48 {\tiny{± 2.90}} & 44.32 {\tiny{± 4.55}} & 55.18 {\tiny{± 0.93}} & 90.50 {\tiny{± 0.21}} & 97.09 {\tiny{± 1.00}} \\
        \midrule
        \rowcolor{gray!20}
        \textbf{MLP+TS} & 39.39 {\tiny{± 0.91}} & 43.86 {\tiny{± 5.23}} & 49.50 {\tiny{± 0.27}} & 66.12 {\tiny{± 0.29}} & 50.97 {\tiny{± 1.56}} \\
        \textbf{GCN+TS} & 44.57 {\tiny{± 2.04}} & 45.27 {\tiny{± 4.74}} & 54.06 {\tiny{± 0.58}} & 91.66 {\tiny{± 0.20}} & 97.80 {\tiny{± 0.20}} \\
        \rowcolor{gray!20}
        \textbf{GAT+TS} & 40.36 {\tiny{± 1.66}} & 44.52 {\tiny{± 3.69}} & \textbf{\textcolor{darkorange!90}{55.73 {\tiny{± 0.31}}}} & 90.93 {\tiny{± 0.20}} & 97.86 {\tiny{± 0.37}}  \\
        \textbf{SAGE+TS} & 41.32 {\tiny{± 2.35}} & 43.32 {\tiny{± 4.67}} & \textbf{\textcolor{customcyan}{56.72 {\tiny{± 0.35}}}} & 91.27 {\tiny{± 0.36}} & 97.33 {\tiny{± 0.94}} \\
        \bottomrule
    \end{tabular}
    }
    \label{tab:hetero_results}
\end{table}

\end{document}